\documentclass{article}

\usepackage{microtype}
\usepackage{graphicx}
\usepackage{subfigure}
\usepackage{booktabs} % for professional tables

\usepackage{hyperref}

\usepackage[accepted]{icml2024}

\usepackage{amsmath}
\usepackage{amssymb}
\usepackage{mathtools}
\usepackage{amsthm}

\usepackage[capitalize,noabbrev]{cleveref}

\theoremstyle{plain}
\newtheorem{theorem}{Theorem}[section]

\newtheorem{lemma}[theorem]{Lemma}

\theoremstyle{definition}
\newtheorem{definition}[theorem]{Definition}

\theoremstyle{remark}
\newtheorem{remark}[theorem]{Remark}

\usepackage[textsize=tiny]{todonotes}

\usepackage{algorithm}
\usepackage{algorithmic}
\usepackage[subtle]{savetrees}

\newtheorem{claim}[theorem]{Claim}

\input{defs.tex}

\newcommand{\E}{\ensuremath{{\mathbb{E}}}}

\icmltitlerunning{Agnostic Mixed Linear Regression with EM and AM}

\begin{document}

\twocolumn[
\icmltitle{Agnostic Learning of Mixed Linear Regressions with \\ EM and AM Algorithms}



\icmlsetsymbol{equal}{*}

\begin{icmlauthorlist}
\icmlauthor{Avishek Ghosh}{yyy}
\icmlauthor{Arya Mazumdar}{comp}
% \icmlauthor{Firstname3 Lastname3}{comp}
\end{icmlauthorlist}

\icmlaffiliation{yyy}{Systems and Control Engg. and Centre for Machine Intelligence for Data Sciences, Indian Institute of Technology, Bombay, India.}
\icmlaffiliation{comp}{Hal\i c\i o\u{g}lu Data Science Institute, University of California, San Diego}
%\icmlaffiliation{sch}{School of ZZZ, Institute of WWW, Location, Country}

\icmlcorrespondingauthor{Avishek Ghosh}{avishek.ghosh38@gmail.com}
%\icmlcorrespondingauthor{Firstname2 Lastname2}{first2.last2@www.uk}

\icmlkeywords{Machine Learning, ICML}

\vskip 0.3in
]



\printAffiliationsAndNotice{}  % leave blank if no need to mention equal contribution

\begin{abstract}
Mixed linear regression is a well-studied problem in parametric statistics and machine learning. Given a set of samples, tuples of covariates and labels, the task of mixed linear regression is to find a small list of linear relationships that best fit the samples. Usually it is assumed that the label is generated stochastically by randomly selecting one of two or more linear functions, applying this chosen function to the covariates, and potentially introducing noise to the result. In that situation, the objective is to estimate the ground-truth linear functions up to some parameter error. The popular expectation maximization (EM)  and alternating minimization (AM) algorithms have been previously analyzed for this.
  
  In this paper, we consider the more general problem of agnostic learning of mixed linear regression from samples, without such generative models. In particular, we show that the AM and EM algorithms, under standard conditions of separability and good initialization, lead to agnostic learning in mixed linear regression by converging to the population loss minimizers, for suitably defined loss functions. In some sense, this shows the strength of AM and EM algorithms that converges to ``optimal solutions'' even in the absence of realizable generative models. 
\end{abstract}

\section{Introduction}
Suppose we obtain samples from a data distribution $\cD$ on $\reals^{d+1}$, i.e., $\{x_i,y_i\}\sim \cD,$ $ 
 x_i \in \reals^d, y_i \in \reals, i =1, \dots, n$.
We
consider the  problem of learning a list of $k$ $\reals^d \to \reals$ linear  functions $y = \theta_j^T x, \theta_j \in \reals^d, j =1, \dots, k$, that best fits  the samples. 
%over $(\cX, \cY)$. The learner is given access to $n$ samples $\{x_i, y_i\}_{i=1}^n$ from the distribution $\cD$. 
%As in usual PAC learning there exists a base function class $\cH: \cX \rightarrow \cY$ and the individual functions of the learned mixture should belong to $\cH$. However, we will work in the paradigm of list decoding where the learner is 
%Suppose we are allowed to output a list of responses given a test sample $x$ each of which corresponds to a mixture component function applied to $x$. We now formally define a list-decodable function class.

This problem is well-studies as the {\em mixed linear regression}, when there are ground-truth $\tilde{\theta}_j, j = 1, \dots, k,$ that generate the samples. For example, 
the setting where 
\begin{align}\label{eq:generative}
x_i \sim \cN(0,I_d), \theta \sim \mathrm{Unif}\{\tilde{\theta_1}, \dots, \tilde{\theta_k}\},  y_i | \theta \sim \cN(x^T\theta, \sigma^2),
\end{align}
for $i =1, \dots, n$ has been analyzed thoroughly. Bounds on sample complexity are provided in terms of $d, \sigma^2$ and error in estimating parameters $\tilde{\theta_j}, j =1, \dots, k$~(\cite{chaganty2013spectral,faria2010fitting,stadler2010l,li2018learning,kwon2018global,viele2002modeling,yi2014alternating,yi2016solving,balakrishnan2017statistical,klusowski2019estimating}).

In this paper, we consider an agnostic and general learning theoretic setup to study the mixed linear regression problem first studied in~\cite{pal2022learning}. In particular, we do not assume a generative model on the samples. Instead we focus on finding the optimal set of lines that minimize a certain loss. %This learning problem does not have to be realizable, i.e. there are not necessarily 

Suppose, we denote a loss function $\ell: \reals^{d \times k}\to \reals$ evaluated on a sample as $\ell(\theta_1, \theta_2, \dots, \theta_k; x,y)$. The population loss is $$\cL(\theta_1, \theta_2, \dots, \theta_k) \equiv \E_{(x,y)\sim \cD}\ell(\theta_1, \theta_2, \dots, \theta_k; x,y),$$ and the population loss minimizers
$$
(\theta^\ast_1, \dots, \theta^\ast_k) \equiv \arg \min 
 \mkern5mu \cL(\theta_1, \theta_2, \dots, \theta_k).
$$

Learning in this setting makes sense if we are allowed to predict a {\em list} (of size $k$) of labels for an input, as pointed out in \cite{pal2022learning}. We may set some  goodness criteria, such as an weighted average of prediction error over all elements in the list. In \cite{pal2022learning}, it was called a `good' prediction if at least one of the labels in the list is good, in particular, the following loss function was proposed, that we will call {\em min-loss}:
\begin{align}\label{eq:minloss}
\ell_{\min}(\theta_1, \theta_2, \dots, \theta_k; x,y) = \min_{j \in [k]} \left \lbrace (y - \inprod{x}{\theta_j})^2  \right \rbrace .
\end{align}
The intuition behind min-loss is simple. Each sample is assigned to a best-fit line, which define a partition of the samples. This is analogous to the popular $k$-means clustering objective.
In addition to the min-loss function, we will also consider the following {\em soft-min} loss function:
\begin{align}\label{eq:softmin}
\ell_{\rm softmin}(\theta_1, \theta_2, \dots, \theta_k; x,y) = \sum_{j=1}^k p_{\theta_1,..,\theta_k}(x,y;\theta_j) \left[ y - \langle x,\theta_j \rangle \right]^2 , 
\end{align}
\begin{align*}	\text{where} \quad p_{\theta_1,..,\theta_k}(x,y;\theta_j) = \frac{e^{-\beta(y - \langle x,\theta_j \rangle)^2}}{\sum_{l=1}^k e^{-\beta(y - \langle x,\theta_l \rangle)^2}}
\end{align*}
with $\beta \ge 0$ as the inverse temperature parameter. Note that, at $\beta \to \infty$, this loss function correspond to the min-loss defined above. On the other hand, at $\beta =0,$ this is simply an average of the squared errors, if a label is uniformly chosen from the list. Depending on how the prediction would occur, the loss function, and therefore the best-fit lines $\theta^\ast_1, \dots, \theta^\ast_k$ will change.

%There is a good reason to consider the soft-min loss. If the samples were to be generated by a process akin to Eq.~\eqref{eq:generative}, then the soft-min objective would have been the {\em negative log-likelihood} function; and therefore minimizing the soft-min loss would have been equivalent to maximum likelihood in that situation (which is asymptotically optimal in parameter estimation).

As is the usual case in machine learning, a learner has access to the distribution $\cD$ only through the samples $\{x_i,y_i\}, i =1, \dots, n$. Therefore instead of the population loss, one may attempt to minimize the empirical loss:
\begin{align*}
    L(\theta_1,\ldots,\theta_k) \equiv \frac{1}{n}\sum_{i=1}^n \ell(\theta_1, \theta_2, \dots, \theta_k; x_i,y_i).
\end{align*}
Usual learning theoretic generalization bounds on excess risk should hold provided the loss function satisfies some properties\footnote{Some discussions on generalization with soft-min loss can be found in Section~\ref{sec:gen_guarantees}.}. However, there are certain caveats in solving the empirical loss minimization problem. For example, even the presumably simple case of squared error (Eq.\eqref{eq:minloss}), the minimization problem is NP-hard, by reduction to the subset sum problem~\cite{yi2014alternating}.

An intuitive and generic iterative method that is widely-applicable for problems with latent variables (in our case, which line is best fit for a sample) is the {\em alternating minimization} (AM) algorithm. At a very high level, starting from some initial estimate of the parameters, the AM algorithm first tries to find a partition of samples according to the current estimate, and then finds the best fit lines within each part. Again under the generative model of \eqref{eq:generative}, AM can approach the original parameters assuming suitable initialization~\cite{yi2014alternating}.

Another popular method of solving mixed regression problems (or in general mixture models) is the well-known {\em expectation maximization} (EM) algorithm.  EM is an iterative algorithm that, starting from an initial estimate of parameters, iteratively update the estimates based on data, by taking an expectation-step and maximization-step repeatedly. For example, it was shown in \cite{balakrishnan2017statistical} that, under the assumption of the generative model that was defined in Eq.~\eqref{eq:generative}, one can give guarantees on recovering the ground-truth parameters  $\tilde{\theta_1}, \dots, \tilde{\theta_k}$ assuming a suitable initialization.

In this paper, we show that the AM and the EM algorithms are in fact more powerful in the sense that even in the absence of a generative model, they lead to agnostic learning of parameters. It turns out, under standard assumptions on data-samples and $\cD$, these iterative methods can output the minimizers of the population loss $\theta^\ast_1, \dots, \theta^\ast_k$ with appropriately defined loss functions. In particular, starting from reasonable initial points, the estimates of the AM algorithm approach $\theta^\ast_1, \dots, \theta^\ast_k$ under the min-loss (Eq.~\ref{eq:minloss}), and the estimates of the EM algorithm approach the minimizers of the population loss under the soft-min loss (Eq.~\ref{eq:softmin}).

Instead of the standard AM (or EM), a version that has been referred to as {\em gradient EM} (and {\em gradient AM}) is also popular and has been analyzed in \cite{balakrishnan2017statistical,zhu2017high, wang2020differentially,pal2022learning} to name a few. Here, in lieu of the maximization step involved in EM (minimization for AM), a gradient step with appropriately chosen step size is taken. This version is amenable to analysis and is strictly worse than the actual EM (or AM) in their generative setting. In this paper as well, we analyze the gradient EM algorithm, and the analogous  gradient AM algorithm.

Recently \cite{pal2022learning} proposed a gradient AM algorithm for the agnostic mixed linear regression problem. However, %there are several technical issues with the results of \cite{pal2022learning}. First, 
they require a strong assumption on initialization %. In particular they  require an initialization 
of $\{\theta_i\}_{i=1}^k$ within a radius of $\mathcal{O}(\frac1{\sqrt{d}})$ of the corresponding $\{\theta^\ast_i\}_{i=1}^k.$ As we can see, in high dimension, the initialization condition is prohibitive. The dimension dependence initialization in \cite{pal2022learning} comes from a discretization ($\epsilon$-net) argument, which was  crucially used to remove inter-iteration dependence of the gradient AM algorithm.

In this paper, we show that a dimension independent initialization is sufficient for gradient AM. In particular, we showed that the initialization needed for $\{\theta_i\}_{i=1}^k$ is $\Theta(1)$, which is a significant improvement over the past work  \cite{pal2022learning}. Instead of an $\epsilon$-net argument, we use fresh samples every round. Moreover, we thoroughly analyze the behavior of restricted covariates on a (problem defined) set, in the agnostic setup, which turns out to be non-trivial. In particular, we observe that the restricted covariates are sub Gaussian with a \emph{shifted mean} and variance, and we need to control the minimum singular value of the covariance matrix of such restricted covariates (which dictates the convergence rate). We leverage some properties of restricted distributions \cite{tallis-mgf}, and were able to analyze such covariates rigorously, obtain bounds and show convergence of AM.

In this paper we also propose and analyze the soft variant of gradient AM, namely gradient EM. As discussed above, the associated loss function is the \emph{soft-min} loss. We show that gradient EM also requires dimension independent $\mathcal{O}(1)$ initialization, and also converges in an exponential rate.

While the performance of both the gradient AM and gradient EM algorithms are similar, AM minimizes a min-loss whereas EM minimizes the optimal soft-min loss (maximum likelihood loss in the generative setup). As shown in the subsequent sections, AM requires a separation condition (appropriately defined in Theorem~\ref{thm:am}) whereas EM does not. On the other hand, EM requires the initialization parameter to satisfy certain condition, albeit mild (exact condition in Theorem~\ref{thm:em}). 
%In numerical simulations (deferred in Appendix~\ref{sec:sim}), we observe that the gradient AM algorithm converges faster than gradient EM.

\subsection{Setup and Geometric Parameters}
Recall that the parameters $\theta^*_1, \ldots,\theta^*_k$ are the minimizers of the population loss function, and we consider both \emph{min-loss} ($\ell_{\min}(.)$) as well as \emph{soft-min} loss ($\ell_{\rm softmin}(.)$) as defined in the previous section. We define
\begin{align*}
    S^*_{j} =  
    \{(x \in \reals^d,y\in \reals): (y - \inprod{x}{\theta^*_j})^2  < (y - \inprod{x}{\theta^*_l})^2,
    %\{ i \in [n]: (y_i - \inprod{x_i}{\theta^*_j})^2  < (y_i - \inprod{x_i}{\theta^*_\ell})^2, 
\end{align*}
$\text{ for all } l \in [k]\setminus j \}$ as the possible set of observations where $\theta^*_j$ is a better (linear) predictor (in $\ell_2$ norm) compared to $\theta^*_1,\ldots,\theta^*_k$. %Note that  $\{S^*_j\}_{j=1}^k$ partition the data $\{x_i,y_i\}_{i=1}^n$ in $k$ sets. 
Furthermore, in order to avoid degeneracy, we assume, for any $j \in [k]$
$$
\Pr_\cD(x: (x,y) \in  S^\ast_j) \ge \pi_{\min},
$$
for some $\pi_{\min}>0.$ We are interested in the probability measure corresponding to the random vector $x$ only, and we integrate (average-out) with respect to $y$ to achieve this. We emphasize that, in the realizable setup, the distribution of $y$ is governed by that of $x$ (and possibly some noise  independent of $x$), and in that setting our definition of $S^*_j$ and $\pi_{\min}$ becomes analogous to that of \cite{yi2014alternating,yi2016solving}\footnote{In \cite{yi2014alternating,yi2016solving}, the authors denote $\{S^*_j\}_{j=1}^k$ as set of indices, but that can be thought of as an analogue to a subset of $\mathbb{R}^{d+1}$ as shown above.}.

% We assume, for all $j,y$,
% $$
% \Pr_{\cD}(S_j^\ast|y) \ge \pi_{\min}.
% $$

% Furthermore, in order to avoid degeneracy, we define, %for any observation $i \in [n]$
% \begin{align*}
% 	\pi_{\min} = %\min_{j \in [k]}\PP(i \in S^*_j) \equiv 
%  \min_{j \in [k]} \frac{|S^*_j|}{n}
% \end{align*}
% where $\pi_{\min} > 0 $. We have the following assumption.

Since we are interested in  recovering $\theta^*_j, j=1, \dots, k$, a few geometric quantities naturally arises in our setup. We define the \emph{misspecification} parameter $\lambda$ as a smallest non-negative number satisfying 
\begin{align*}
    |y_i - \langle x_i, \theta^*_j \rangle| \leq \lambda \quad \text{for all } (x_i,y_i) \in S^*_j \quad \text{and } j \in[k].
\end{align*}
Moreover, we also define the \emph{separation} parameter $\Delta$ as the largest non-negative number satisfying
\begin{align*}
    \min_{l \in [k]\setminus j} |y_i - \langle x_i, \theta^*_l \rangle| \geq \Delta \quad \text{for all } (x_i,y_i) \in S^*_j.
\end{align*}
% \begin{assumption}
% \label{asm:main_asm}
% 	For all $(x_i,y_i) \in S^*_j$, we have
% 	$|y_i - \langle x_i, \theta^*_j \rangle| \leq \lambda,
% 	$ where $\lambda > 0$ is the misspecification parameter. Furthermore, for all $(x_i,y_i) \in S^*_j$, we have $\min_{\ell \in [k]\setminus j} |y_i - \langle x_i, \theta^*_\ell \rangle| \geq \Delta$, with $\Delta$ as the separation parameter.
% \end{assumption}
Let us comment on these geometric quantities. Note that in the case of a realizable setup, the parameter $\lambda = 0$ in the noiseless case or proportional to the noise in the noisy case. In words, $\lambda$ captures the level of misspecification from the linear model.
On the other hand, the parameter $\Delta$ denotes the separation or margin in the problem. In classical mixture of linear regression framework, with realizable structure, similar assumptions are present in terms of the (generative) parameters. Moreover, with the realizable setup, our assumption can be shown to be \emph{exactly} same as the usual separation assumption.

\subsection{Summary of Contributions}
Let us now describe the main results of the paper. To simplify exposition, we state the results here informally and the rigorous statements may be found in Sections~\ref{sec:soft_loss} and \ref{sec:am}.

Our main contribution is analysis of the gradient AM and gradient EM algorithms. The gradient AM algorithm works in the following way. At iteration $t$, based on the current parameter estimates $\{\theta^{(t)}_j\}_{j=1}^k$, the gradient AM algorithm constructs estimates of $\{S^*_j\}_{j=1}^k$, namely $\{S^{(t)}_j\}_{j=1}^k$. The next iteration is then obtained by taking a gradient (with $\gamma$ as step size) over the quadratic loss over all such data points $\{i: (x_i,y_i) \in S^{(t)}_j\}$ for all $j \in [k]$.

On the other hand, in the $t$-th iteration, the gradient EM algorithm uses the current estimate of $\{\theta^*_j\}_{j=1}^k$, namely $\{\theta^{(t)}_j\}_{j=1}^k$ to compute the \emph{soft-min} probabilities $p_{\theta^{(t)}_1,\ldots,\theta^{(t)}_k}(x_i,y_i;\theta^{(t)}_j)$ for all $j \in [k]$ and $i \in [n]$. Then, using these probabilities, the algorithm takes a gradient of the \emph{soft-min} loss function with step size $\gamma$ to obtain the next iteration.

We begin by assuming the covariates $x_i \stackrel{i.i.d}{\sim} \cN(0,I_d)$. Note that this assumption serves as a natural starting point of analyzing several EM and AM algorithms (\cite{balakrishnan2017statistical,yi2014alternating,yi2016solving,netrapalli2015phase,ghosh2020alternating}). Furthermore, as stated earlier, we emphasize that in order to obtain convergence, we need to understand the behavior of restricted covariates in the agnostic setting. 
%In \cite{pal2022learning}, the authors incorrectly assumed the behavior of the restricted and unrestricted covariates to be the same. 
We require Gaussians, because the behavior of restricted Gaussians are well studied in statistics \cite{tallis-mgf} and we use several such classical results.

We first consider the min-loss and employ the gradient AM algorithm, similar to \cite{pal2022learning}. In particular, we show that the iterates returned by the gradient AM algorithm after $T$ iterations, $\{\theta_j^{(T)}\}_{j=1}^k$ satisfy
\begin{align*}
    \|\theta_j^{(T)} - \theta^*_j\| \leq \rho^T \| \theta_j^{(0)} - \theta^*_j\| + \delta,
\end{align*}
with high probability (where $\rho <1$) provided $n$ is large enough and $ \| \theta_j^{(0)} - \theta^*_j\| \leq c_{\mathsf{ini}}\|\theta^*_j\|$. Here $c_{\mathsf{ini}}$ is the initialization parameter and $\delta$ is the error floor that stems from the agnostic setting and the gradient AM update (see \cite{balakrishnan2017statistical} where, even with generative setup, an error floor is shown to be unavoidable). Here $\delta$ depends on the step size of the gradient AM algorithm as well as the several geometric properties of the problem like misspecification and separation. 
However, the result of \cite{pal2022learning} in this regard requires an initialization of $\{\theta_i\}_{i=1}^k$ within a radius of $\mathcal{O}(\frac1{\sqrt{d}})$ of the corresponding $\{\theta^\ast_i\}_{i=1}^k$ which we improve on. 
%Moreover,   the convergence proof \cite{pal2022learning} implicitly uses  restricted covariates (on a problem dependent set) without explicitly characterizing their distribution \footnote{In the generative setup with Gaussian covariates, it is shown (in \cite{yi2016solving}) that the distribution of the  restricted covariates are sub-Gaussian with zero mean and constant sub-Gaussian parameter, resulting similar concentration bounds. However, it is not clear whether these claims  hold in the agnostic setting.}. %Furthermore, the \emph{soft-min} loss and the associated gradient EM algorithm was not covered in \cite{pal2022learning}.

% As mentioned earlier, in order to show convergence of gradient AM, the authors require a lower bound on the minimum singular value of $\frac{1}{n}\sum_{i \in S} x_i x_i^T$, where $\|x_i\| \leq 1$ (where the set $S$ is problem dependent). Moreover, they require a concentration on $\langle x_i,u \rangle$ where $u$ is some vector. Now, the concentration results on both the cases require the analysis of $x_i$ conditioned on the event that $i \in S$. In the generative setup with Gaussian covariates, it is shown (in \cite{yi2016solving}) that the conditioned covariates are sub-Gaussian with zero mean and constant sub-Gaussian parameter. However, these claims do not hold in the agnostic setting, and the conditioned covariates may have a non-zero mean. Furthermore, it is not clear whether they remain sub-Gaussian with constant parameter. The authors of \cite{pal2022learning}  just used the results of \cite{yi2016solving}, rendering the proof incorrect.

In this paper, we show that it suffices for the initial parameters to be within a (constant) $\Theta(1)$ radius for convergence, provided the geometric parameter $\Delta-\lambda$ is large enough. The $\Theta(1)$ initialization matches the standard (non agnostic, generative) initialization for mixed linear regression (see \cite{yi2014alternating,yi2016solving}). In order to analyze the gradient AM algorithm we need to characterize the behavior of covariates $\{x_i\}_{i=1}^n$ restricted to sets $\{S^*_j\}_{j=1}^k$. In particular we need to control the norm of such restricted Gaussians as well as control the minimum singular value of a random matrix whose rows are made of such random variables.
%\footnote{In the generative setup with Gaussian covariates, it is shown (in \cite{yi2016solving}) that the distribution of the  restricted covariates are sub-Gaussian with zero mean and constant sub-Gaussian parameter, resulting similar concentration bounds. However, it is not clear whether these claims  hold in the agnostic setting.}. 
Specifically, we require (i) a lower bound on the minimum singular value of $\frac{1}{n}\sum_{x_i \in S} x_i x_i^T$, where  the set $S$ is problem dependent, (ii) an upper bound on $\|x_i\|$ where $x_i \in S$ and (iii) a concentration on $\langle x_i,u \rangle$ where $u$ is some vector and $x_i \in S$.

In order to obtain the above,  we leverage the properties of restricted Gaussians (\cite{tallis-mgf,ghosh2019max}) on a (generic) set with Gaussian volume bounded away from zero and show that the resulting distribution of the covariates is sub Gaussian with non-zero mean and constant parameter. We obtain upper bounds on the shift and the sub Gaussian parameter. We would like to emphasize that in the realizable setup of mixed linear regressions, as shown in \cite{yi2014alternating,yi2016solving} such a characterization may be obtained with lesser complication. However, in the agnostic setup, it turns out to be quite non-trivial.

Moreover, in gradient AM, the setup is complex since the sets are formed by the current iterates of the algorithm (and hence random), unlike  $\{S^*_j\}_{j=1}^k$, which are fixed. In order to handle this, we employ re-sampling in each iteration to remove the inter-iteration dependency. We would like to emphasize that sample splitting is a standard technique in the analysis of AM type algorithms and several papers (e.g. \cite{yi2014alternating,yi2016solving,ghosh2020alternating} for mixed linear regression, \cite{netrapalli2015phase} for phase retrieval and \cite{ghosh2020efficient} for distributed optimization) employ such a technique. While this is not desirable, this is a way to remove the inter iteration dependence that comes through data points. Finer techniques like leave-one-out analysis (LOO) is also used (\cite{chen2019gradient}) but for simpler problems (like phase retrieval) since the LOO updates are quite non-trivial. This problem exaggerates further in the agnostic setup. Hence, as a first step, in this paper we assume a simpler sample split based framework and keep finer techniques like LOO as future direction.

We would also like to take this opportunity to correct an error in \citep[Theorem 4.2]{pal2022learning}. In particular, that theorem  should hold only for  Gaussian covariates,  not for general bounded covariates as stated. It was incorrectly assumed in that paper that the lower bound on the singular value mentioned above holds for general covariates.

We then move on to analyze the \emph{soft-min} loss and analyze the gradient EM algorithm. Here, we show similar contraction guarantees in the parameter space as in gradient EM. There are several technical difficulties that arise in the analysis of the gradient EM algorithm for agnostic mixed linear regressions-- (i) First, we show that if $(x_i,y_i) \in S^*_j$, then the soft-min probability $p_{\theta^*_1,\ldots,\theta^*_k}(x_i,y_i;\theta^*_j) \geq 1- \eta$, where $\eta$ is small. (ii) Moreover, using the initialization condition, and the properties of the soft-max function (\cite{gao2017properties}) we argue that $p_{\theta^{(t)}_1,\ldots,\theta^{(t)}_k}(x_i,y_i; \theta^{(t)}_j)$ is close to $p_{\theta^*_1,\ldots,\theta^*_k}(x_i,y_i;\theta^*_j)$, where $\{\theta^{(t)}_j\}_{t=1}^T$ are the updated of the gradient EM algorithm.

Our results for agnostic gradient AM and EM consist some extra challenge over the existing results in literature (\cite{balakrishnan2017statistical,waldspurger2018phase}). Usually, the population operator with Gaussian covariates are analyzed (mainly in EM, see \cite{balakrishnan2017statistical}), and then a finite sample guarantee is obtained using concentration arguments. However, in our setup, with the soft-min probabilities and the $\min$ function, it is not immediately clear how to analyze the population operator. Second, in the gradient EM algorithm, we do not split the samples over iterations, and necessarily handle the inter-iteration dependency of covariates.

Furthermore, to understand the \emph{soft-min} and \emph{min} loss better, in Section~\ref{sec:gen_guarantees}, we obtain generalization guarantees that involve computing the Rademacher complexity of such function classes. Agreeing with intuition,  the complexity of  \emph{soft-min} and \emph{min} loss class is at most $k$ times the complexity of the learning problem of simple linear regression with quadratic loss. %The generalization bound implies that the solution of the empirical risk minimization is c 

% Finally, in Appendix~\ref{sec:sim}, we implement both the algorithms. We observe that, while in theory we require fresh samples at each iteration (sample splitting), it can be avoided in practice. Moreover, even without \emph{good} initialization, both gradient AM and EM algorithms converge with multiple restarts. %We verify this by taking the \emph{best} out of $5$ trials. 
% In the generative setup, this phenomenon is common (see \cite{yi2014alternating,yin2018learning,ghosh2019max,ghosh2020alternating}). We observe that the same thing happens in the agnostic case as well.

\subsection{Related works}
As discussed earlier, most works on the mixture of linear regressions are in the realizable setting, and aim to do parameter estimation. Algorithms like EM and AM are most popularly used to achieve this task. For instance, in \cite{balakrishnan2017statistical}, it was proved that a \emph{suitable initialized} EM algorithm is able to find the correct parameters of the mixed linear regressions. Although \cite{balakrishnan2017statistical} obtains the convergence results within an $\ell_2$ ball, it is then extended to an appropriately defined cone by \cite{klusowski2019estimating}. On the AM side, \cite{yi2014alternating} introduced the AM algorithm for the mixture of $2$ regressions, where the initialization is done by the spectral methods. Then, \cite{yi2016solving} extends that to a mixture of $k$ linear regressions. Perhaps surprisingly, for the case of $2$ lines, \cite{kwon2018global} shows that any random initialization suffices for EM algorithm to converge. In the above mentioned works, the covariates are assumed to be standard Gaussians, which was relaxed in \cite{li2018learning}, allowing Gaussian covariates to have different covariances. Here, near optimal sample as well as computational complexities were achieved albeit not via EM or AM type algorithm.

In another line of work, the convergence rates of AM or its close variants are investigated. In particular, in \citep{ghosh2020alternating,shen2019iterative}, it is shown that AM (or its variants) converge at a double-exponential (super-linear) rate. Recent work, \cite{chandrasekher2021sharp} shows similar results for larger class of problems.

We emphasize that apart from mixture of linear regressions, EM or AM type algorithms are used to address other problems as well. Classically parameter estimation in the mixture of Gaussians is done by EM  mixture of Gaussians (see \cite{balakrishnan2017statistical,daskalakis2014faster} and the references therein). The seminal paper by \cite{balakrishnan2017statistical} addresses the problem of Gaussian mean estimation as well as linear regression with missing covariates. Moreover, AM type algorithms are used in phase retrieval (\cite{netrapalli2015phase,waldspurger2018phase}), parameter estimation in max-affine regression (\cite{ghosh2019max}), clustering in distributed optimization (\cite{ghosh2020efficient}).

In all of the above mentioned works, the covariates are given to the learner. However, there is another line of research that focuses on analyzing AM type algorithms when the learner has the freedom to design the covariates (\cite{yin2018learning,kris2019sampling,mazumdar2020recovery,mazumdar2022learning,pal2021support}).

However, none of these works is directly comparable to our setting. All these works assume a realizable model where the parameters come with the problem setup. However, ours is an agnostic setup, and here there are no optimal parameters associated with the setup, rather solutions of (naturally emerging) loss functions.

Our work is a direct follow up of \cite{pal2022learning}, who introduced the agnostic learning framework for mixed linear regression, and also used the AM algorithm in lieu of empirical risk minimization. Also, \cite{pal2022learning} only considered the min-loss, and neither the soft-min loss nor the EM algorithm, whereas we consider both EM and AM. Moreover, the AM guarantees we obtain are sharper than that of \cite{pal2022learning}.

\subsection{Organization}
We start with the \emph{soft-min} loss function and the gradient EM algorithm in Section~\ref{sec:soft_loss}. In Section~\ref{sec:em_theory}, we obtain the theoretical results of gradient EM. We then move to \emph{min} loss function in Section~\ref{sec:am}, where we analyze the gradient AM algorithm, with theoretical guarantees given in Section~\ref{sec:am_theory}. We present a rough overview of the proof techniques in Section~\ref{sec:proof_sketch}. Finally, in Section~\ref{sec:gen_guarantees}, we provide some generalization guarantees   using Rademacher complexity. We conclude in Section~\ref{sec:con} with a few open problems and future direction. We collection all the proofs (both EM and AM) in Appendix~\ref{app:em} and \ref{app:am}. 

\subsection{Notation}
Throughout this paper, we use $\|.\|$ to denote the $\ell_2$ norm of a $d$ dimensional vector unless otherwise specified. Also for a positive integer $r$, we use $[r]$ to denote the set $\{1,\ldots,r\}$. We use $C,C_1,C_2,\ldots,c,c_1,c_2\ldots $ to denote positive universal constants, the value of which may differ from instance to instance.

%We will omit the $k$ in $\bar{\cH}$ when clear from context. As in PAC learning, we need to define a loss measure to quantify the quality of learning. For the list decodable setting, it is natural to compete on the minimum loss achieved by any of the values in the list for a particular example~\citep{kothari2018robust}. In order to formally define the min-loss setting, let $\ell: \cY \times \cY \rightarrow \reals^+$ be a base loss function. Then the min-loss is defined as follows,
% \begin{align*}
%     \cL(y, \bar{h}(x)) &:= \min_{j \in [k]} \ell(y, \bar{h}(x)_j) = \min_{j \in [k]} \ell(y, h_j(x)) \\
%     L(\bar{h}) &:= \frac{1}{n} \sum_{i=1}^n \cL(y_i, \bar{h}(x_i)). \numberthis \label{eq:minloss}
% \end{align*}
% %and our objective is to minimize 

% In much of this paper we will specialize to a setting where the base function class is linear i.e $\cH = \{\inner{\theta}{\cdot}: \forall \theta \in \reals^{d} \text{ s.t } \norm{\theta}_2 \leq w\}$. In this case, we can follow a simplified notation for the min-loss,
% \begin{align}\label{eq:loss}
%     L(\theta_1,\ldots,\theta_k) = \frac{1}{n}\sum_{i=1}^n  \min_{j \in [k]} \left \lbrace (y_i - \inprod{x_i}{\theta_j})^2  \right \rbrace.
% \end{align}
% \vspace{-3mm}
% \begin{align}
% \label{eq:erm}
% \text{with} \,\,\, \, (\theta_1^*, \ldots, \theta_k^*) = \argmin_{\{\theta_j\}_{j=1}^k} L(\theta_1,\ldots,\theta_k).
% \end{align}
\section{Agnostic Mixed Linear Regression-Min-Loss}
\label{sec:am}
In this section, we analyze the min-loss function and analyze gradient AM algorithm. First, recall the definition of $\ell_{\min}(.)$ from Eq.~\ref{eq:minloss}.
Similar to the section above, we are given a set of $n$ data-points $\{x_i,y_i\}_{i=1}^n$, where $x_i \in \mathbb{R}^d$ and $y_i \in \mathbb{R}$ drawn from an unknown distribution $\cD$. We want to obtain
\begin{align*}
    (\theta^*_1,\ldots,\theta^*_k) = \mathrm{argmin} \,\, \mathbb{E}_{(x,y)\sim \mathcal{D}}\ell_{\text{min}} (\theta_1,\ldots,\theta_k; x,y).
\end{align*}
 With the given $n$ datapoints, we aim to learn these $k$ hyperplanes via the AM algorithm (Algorithm~\ref{alg:am_improved}), which tries to minimize the empirical optimization version instead. 
 %We proceed to compute the minimizers of the \emph{empirical} loss function, given by
% \begin{align*}
% 	L(\theta_1,..,\theta_k) = \frac{1}{n}\sum_{i=1}^n \ell_{\text{min}} (\theta_1,\ldots,\theta_k; x_i,y_i),
% \end{align*}

% Recall the definition of $S^*_j$ for all $j\in [k]$, which partitions the dataset in $k$ disjoint parts, where 
% $$S^*_j = \{ i \in [n]: (y_i - \inprod{x_i}{\theta^*_j})^2  < (y_i - \inprod{x_i}{\theta^*_\ell})^2\}, 
%  \text{for all } \ell \in [k]\setminus j.$$
%  In words, $S^*_1$ is the set of observations, where $\theta^*_1$ is a better (linear) predictor compared to $\theta^*_2,\ldots,\theta^*_k$.

%  Furthermore, we define, for any observation $i \in [n]$
% \begin{align*}
% 	\pi_{\min} = \min_{j \in [k]}\PP(i \in S^*_j)
% \end{align*}
% where $\pi_{\min} > 0 $.

% \begin{assumption}
% For all $i \in S^*_j$, we have
% $|y_i - \langle x_i, \theta^*_j \rangle| \leq \lambda,
% $ where $\lambda > 0$ is small. Furthermore, for all $i \in S^*_j$, we have $\min_{\ell \in [k]\setminus j} |y_i - \langle x_i, \theta^*_\ell \rangle| \geq \Delta$, with $\Delta > 0$ is the separation of the problem.
% \end{assumption}

\subsection{Gradient AM Algorithm}
In this section we use the gradient AM algorithm for minimizing $L(\theta_1,\ldots,\theta_k)$. The details of our algorithm is given in Algorithm~\ref{alg:am_improved}. 
%Note that similar algorithm featured in \cite{pal2022learning}. Here we explain the differences between the algorithms briefly. 

First note that here, we split the $n$ samples $\{x_i,y_i\}_{i=1}^n$ into $2T$ disjoint samples where we run Algorithm~\ref{alg:am_improved} for $T$ iterations. We would like to remind that sample splitting is a standard in AM type algorithms  (\cite{yi2014alternating,yi2016solving,ghosh2020alternating,netrapalli2015phase,ghosh2020efficient}). While this is not desirable, this is a way to remove the inter iteration dependence that comes through data points.

Hence, at each iteration of gradient AM we are given $n' = n/2T$ samples. Each iteration consists of $2$ stages (see Algorithm~\ref{alg:am_improved}). In the first stage of the $t$-th iteration, we use 
$n'$ samples to construct the index sets $I^{(t)}_j$  in the following way
\begin{align*}
     I^{(t)}_j &= \lbrace i \in [n']: (y_i ^{(t)}- \inprod{x_i^{(t)}}{\theta^{(t)}_j})^2 
      <  (y_i ^{(t)}- \inprod{x_i^{(t)}}{\theta^{(t)}_{j'}})^2   \rbrace
 \end{align*}
$\forall \,\, j' \in [k]\setminus j$. Here, we collect the data points for which the current estimate of $\theta^*_j$, namely $\theta^{(t)}_j$ is a better (linear) estimator than $\{\theta^{(t)}_{j'}\}$ where $j'\neq j$. Notw that $\{ I^{(t)}_j\}_{j=1}^k$ partitions $[n']$.

At the second stage of gradient AM, we use another set of fresh $n'$ data points to run the gradient update on the set $\{ I^{(t)}_j\}_{j=1}^k$ with step size $\gamma$ to obtain the next iterate $\{\theta^{(t+1)}_j\}_{j=1}^k$. The details is given in Algorithm~\ref{alg:am_improved}.

\begin{algorithm}[t!]
  \caption{Gradient AM for Mixture of Linear Regressions}
  \begin{algorithmic}[1]
 \STATE  \textbf{Input:} $\{x_i,y_i\}_{i=1}^n$, Step size $\gamma$
 \STATE \textbf{Initialization:} Initial iterate $\{\theta^{(0)}_j\}_{j=1}^k$  \\
 \STATE Split all samples into $2T$ disjoint datasets $\{ x_i ^{(t)},y_i^{(t)}\}_{i=1}^{n'}$ with $n' = n/2T$ for all $t =0,1,\ldots,T-1$
  \FOR{$t=0,1, \ldots, T-1 $}
\STATE \underline{Partition:}  \\
 \STATE For all $j\in[k]$, use $n'$ samples to construct index sets $\{I_j^{(t)}\}_{j=1}^k$ such that $\forall \,\, j' \in [k]\setminus j$,
 \vspace{-2mm}
 \begin{align*}
     I^{(t)}_j &= \lbrace i: (y_i ^{(t)}- \inprod{x_i^{(t)}}{\theta^{(t)}_j})^2 
      <  (y_i ^{(t)}- \inprod{x_i^{(t)}}{\theta^{(t)}_{j'}})^2 \rbrace
 \end{align*}
\STATE \underline{Gradient Step:}
\STATE Use fresh set of $n'$ samples to run gradient update
\begin{align*}
    \theta^{(t+1)}_j = \theta^{(t)}_j - \frac{\gamma}{n}\sum_{i \in [n']} \nabla F_i(\theta^{(t)}_j) \, \mathbf{1}\{i \in I^{(t)}_j\},  \,\,\, \forall \,\, j \in [k]
\end{align*}
\vspace{-2mm}
  \STATE where $F_i(\theta^{(t)}_j) = (y_i ^{(t)}- \inprod{x_i^{(t)}}{\theta^{(t)}_j})^2$
  \ENDFOR
  \STATE \textbf{Output:} $\{\theta^{(T)}_j\}_{j=1}^k$
 \end{algorithmic}
  \label{alg:am_improved}
\end{algorithm}

\subsection{Theoretical Guarantees}
\label{sec:am_theory}
In this section, we obtain theoretical guarantees for Algorithm~\ref{alg:am_improved}. Similar to the previous section, we assume $|y_i| \leq b$ for all $i \in [n]$. 
In the following, we consider one iteration of Algorithm~\ref{alg:am_improved}, and show a contraction in parameter space. Let the current parameter estimates are $\{\theta_j\}_{j=1}^k$ and  the corresponding to the index $\{I_j\}_{j=1}^k$. Moreover, let the next iterates are $\{\theta^+_j\}_{j=1}^k$. Unpacking, the next iterate is given by
\begin{align}
    \theta^+_j = \theta_j - \frac{2\gamma}{n} \sum_{i \in I_j} [x_i x_i^T \theta_j - y_i x_i]
    \label{eqn:am_update}
\end{align}
for all $j \in [k]$. We now present our main results of this section.
\begin{theorem}[Gradient AM]
\label{thm:am}
Suppose $x_i \stackrel{i.i.d} {\sim} \cN(0,I_d)$ and that  $n' \geq C\frac{d \log(1/\pi_{\min})}{\pi_{\min}^3}$.  Furthermore,
\begin{align*}
    \|\theta_j - \theta^*_j \| \leq c_{\mathsf{ini}} \|\theta^*_j\|
\end{align*}
for all $j \in [k]$ where $c_{\mathsf{ini}}$ is a small positive constant (initialization parameter). Moreover, let the separation parameter satisfy
\begin{align*}
    \Delta >  \lambda + C_1 \, [c_{\mathsf{ini}} \sqrt{\log(1/\pi_{\min}}) \max_{j\in [k]}  \|\theta^*_j\| + \sqrt{1+ \log(1/\pi_{\min})}].
\end{align*}
Then, running one iteration of Gradient AM with step size $\gamma $, yields $\{\theta^{+}_j\}_{j=1}^k$ satisfying
\begin{align*}
    \|\theta^+_j - \theta^*_j\| & \leq \rho \|\theta_j - \theta^*_j\| + \varepsilon, \quad \text{with probability exceeding}
    % c_1 \, \, \mu \, \|\theta^*_j\| \\
    % & \qquad + c \,k\, \exp \left( - c_1 \frac{(\Delta-2\lambda)^2}{(\max_j \|\theta^*_j\| )^2}\, d \right) \|\theta^*_j\|,
\end{align*}
$1-C_1 \exp(-C_2 \pi_{\min}^4 n')  -c_1\exp(-P_e n')-\frac{n'}{\mathsf{poly}(d)}$, where $\rho = (1- c \gamma \pi_{\min}^3)$, and the error floor
\begin{align*}
\varepsilon &\leq  C \gamma \lambda \sqrt{d\log d \log(1/\pi_{\min})}+C_1 \gamma (k-1) P_e \\
& \times \left[ d\log d \log(1/\pi_{\min}) \|\theta^*_1\| + C b \sqrt{d\log d \log(1/\pi_{\min})}\right],
\end{align*}
\begin{align*}
	\text{and  } P_e \leq 4 \exp \bigg( - \frac{1}{c_{\mathsf{ini}^2} \max_{j \in [k]}\|\theta^*_j\|^2} [ \frac{\Delta -\lambda}{2} ]^2 \bigg ).
\end{align*}
\end{theorem}
The proof of Theorem~\ref{thm:am} is deferred to Appendix~\ref{app:am}. We make a few remarks here.
\begin{remark}[Contraction factor $\rho$]
    We observe that if $\rho <1$, the above result implies a contraction in parameter space with a slack of $\varepsilon$, which we call the error-floor. Note that by choosing $\gamma < \frac{c_0}{ (1-\eta) \pi_{\min}^3}$, where $c_0$ is a small constant, we can always make $\rho <1$.
\end{remark}
\begin{remark}[Error floor $\varepsilon$]
Observe that the error floor $\varepsilon$ depends linearly on the step size $\gamma$, similar to any standard stochastic optimization problem. The error floor also decays linearly with the misspecification parameter $\lambda$, which may be thought as an agnostic bias. In previous works \cite{yi2016solving,yi2014alternating}, even in the realizable setting, either the authors assume $\lambda =0$ or very small. In a related field of online learning (multi armed bandits and reinforcement learning in linear framework), this model misspecification also impacts the regret in a linear fashion as seen by \citep[Theorem 5]{jin20a}. Even in these realizable setting, is it unknown how to tackle large $\lambda$.
\end{remark}
% \begin{remark}[Exact convergence to $\{\theta^*_j\}_{j=1}^k$ with no-error floor]
%    In  \citep[Remark 4.5]{pal2022learning} exact convergence to optimal parameters was incorrectly claimed. In \citep[Theorem 4.2]{pal2022learning}, note that the error floor becomes the dominating term when the iterate is close to $\{\theta^*_j\}_{j=1}^k$. With a little calculation, one can immediately see that the condition on the misspecification parameter becomes vacuous when the iterates are close to $\{\theta^*_j\}_{j=1}^k$, thus making the claim vacuous.
% \end{remark}
% \begin{remark}[Comparison with \cite{pal2022learning}]
% The initialization condition requires $\theta_j$ to be within a norm ball of radius $\mathcal{O}(1)$, whereas \cite{pal2022learning} requires $\mathcal{O}(1/\sqrt{d})$ which is much more stringent. Also, the $\mathcal{O}(1)$ initialization matches the standard (non agnostic, generative) initialization for mixed linear regression (see \cite{yi2014alternating,yi2016solving}).
% \end{remark}
\begin{remark}[Re-sampling]
    Note that the gradient AM algorithm of ours requires re-sampling fresh data points in every iteration. Similar to the analysis of the gradient EM, here also we need to control the lower spectrum of a random matrix consisting Gaussians restricted to a set. From the structure of gradient AM, this set here is given by $S^{(t)}_j = \{ (x_i,y_i): i \in I^{(t)}_j \}$. Note that without re-sampling of data points, analyzing the behavior of Gaussians on the sets $\{S^{(t)}_j\}_{j=1}^k$ turns out to be quite non-trivial since $\{S^{(t)}_j\}_{j=1}^k$ depends on $\{\theta^{(t)}_j\}_{j=1}^k$ which depends on all the data point $\{x_i,y_i\}_{i=1}^n$. 
\end{remark}
\begin{remark}[Probability of error $P_e$]
   One major part in showing the convergence guarantee is to show that provided good initialization, the probability of a datapoint lying in an incorrect index set is at most $P_e$. With a closer look, it turns out that if the problem is separated enough ($\Delta$ large) and the initialization is suitable ($c_{\mathsf{ini}}$ is small), $P_e$ decays exponentially fast. Hence, in such a setup, the second term in $\varepsilon$ is quite small.
\end{remark}
\begin{remark}[Sample complexity]
   Note that we require the number of samples satisfying the following: $n \geq C \,\, \frac{d \log (1/\pi_{\min})}{\pi_{\min}^3}$, where the dependence on $k$ comes through $\pi_{\min}$ (and from definition, we have $\pi_{\min} \leq 1/k$). Note that information theoretically, we only require $\Omega(kd)$ samples, since there are $kd$ unknown parameters to learn. Hence, our sample complexity is optimal in $d$. However, it is sub-optimal in $k$ compared to the standard (non-agnostic) AM guarantees (\cite{yi2014alternating,yi2016solving}). The sub-optimality comes from the proof techniques we use for the agnostic setting. In particular, we use spectral properties of a restricted Gaussian vectors on a set with (Gaussian) volume at least $\pi_{\min}$. As shown in \cite{ghosh2019max}, this gives rise to a dependence of $1/\pi_{\min}^3$ in sample complexity. Moreover, in \cite{ghosh2019max}, it is argued (albeit in a different problem), that when spectral properties of such restricted Gaussians are employed,  a $1/\pi_{\min}^3$ dependency is in general unavoidable.
\end{remark}
\section{EM algorithm for Soft-Min Loss}
\label{sec:soft_loss}
In this section we analyze the \emph{soft-min} loss function and propose gradient EM algorithm to address this. Recall the definition of $\ell_{\text{softmin}}(.)$ from Eq.~\ref{eq:softmin}. Moreover, recall that we are given a set of $n$ data-points $\{x_i,y_i\}_{i=1}^n$, where $x_i \in \mathbb{R}^d$ and $y_i \in \mathbb{R}$ drawn from an unknown distribution $\cD$. Our goal here is to obtain
\begin{align*}
    (\theta^*_1,\ldots,\theta^*_k) = \mathrm{argmin} \,\, \mathbb{E}_{(x,y)\sim \mathcal{D}}\ell_{\text{softmin}} (\theta_1,\ldots,\theta_k; x,y).
\end{align*}
We aim to learn these $k$ hyperplanes through the given data. The EM algorithm (Algorithm~\ref{alg:grad-em}) tries to minimize the empirical version of the problem.

%Hence, we proceed to compute the minimizers of the \emph{empirical} loss function, given by
% \begin{align*}
% 	L(\theta_1,..,\theta_k) = \frac{1}{n}\sum_{i=1}^n \ell_{\text{softmin}} (\theta_1,\ldots,\theta_k; x_i,y_i),
% \end{align*}
\begin{algorithm}[t!]
	\caption{Gradient EM for Mixture of Linear Regressions}
	\begin{algorithmic}[1]
		\STATE  \textbf{Input:} $\{x_i,y_i\}_{i=1}^n$, Step size $\gamma$ 
		\STATE \textbf{Initialization:} Initial iterate $\{\theta^{(0)}_j\}_{j=1}^k$  \\
		\FOR{$t=0,1, \ldots, T-1 $}
		\STATE \underline{Compute Probabilities:}  \\
		\STATE Compute $p_{\theta^{(t)}_1,..,\theta^{(t)}_k}(x_i,y_i;\theta^{(t)}_j)$ for all $j \in [k]$ and $i \in [n]$
		\STATE \underline{Gradient Step:} (for all $j \in [k]$)
		\vspace{-1mm}
		\begin{align*}
			\theta^{(t+1)}_j = \theta^{(t)}_j - \frac{\gamma}{n}\sum_{i =1}^n p_{\theta^{(t)}_1,..,\theta^{(t)}_k}(x_i,y_i;\theta^{(t)}_j) \nabla F_i(\theta^{(t)}_j), 
		\end{align*}
		\vspace{-2mm}
		\STATE where $F_i(\theta^{(t)}_j) = (y_i - \inprod{x_i}{\theta^{(t)}_j})^2$
		\ENDFOR
		\STATE \textbf{Output:} $\{\theta^{(T)}_j\}_{j=1}^k$
	\end{algorithmic}
	\label{alg:grad-em}
\end{algorithm}

\subsection{Gradient EM Algorithm}
We propose EM based algorithm for minimizing the empirical loss function $L(\theta_1,..,\theta_k)$. In particular we propose a variant of EM, popularly known as gradient EM for this. The steps are given in Algorithm~\ref{alg:grad-em}. Each iteration of gradient EM consists of two steps. First, in the compute probability step, based on the current estimates of $\{\theta^*_j\}_{j=1}^k$, namely $\{\theta^{(t)}\}_{j=1}^k$, Algorithm~\ref{alg:grad-em} computes the soft-min probabilities computed using the current iterates $\{\theta^{(t)}\}_{j=1}^k$, which is $p_{\theta^{(t)}_1,\ldots,\theta^{(t)}_k}(x_i,y_i;\theta^{(t)}_j)$ for all $j \in [k]$ and $i \in [n]$. In the subsequent step, using these probabilities, the algorithm takes a gradient step with step size $\gamma$. In particular, for the $j$-th iterate $\theta_j^{(t)}$, gradient EM weights the standard quadratic loss computed on the $i$-th data point, given by $(y_i - \inprod{x_i}{\theta^{(t)}_j})^2$ and takes the gradient to obtain the next iterate $\{\theta^{(t+1)}_j\}_{j=1}^k$. We truncate Algorithm~\ref{alg:grad-em} after $T$ steps.

We split the $n$ samples $\{x_i,y_i\}_{i=1}^n$ into $2T$ disjoint samples where we run Algorithm~\ref{alg:grad-em} for $T$ iterations. Again sample splitting is a standard in EM type algorithms  (\cite{balakrishnan2017statistical,kwon2018global}). Hence, at each iteration of gradient EM we are given $n' = n/2T$ samples. Each iteration consists of $2$ stages (see Algorithm~\ref{alg:grad-em}). The first $n'$ samples are used to compute the probabilities, and the next set of samples are used to take the gradient step.

\subsection{Theoretical Guarantees}
\label{sec:em_theory}
We now look at the convergence guarantees of Algorithm~\ref{alg:grad-em}. In particular, here we consider one iterate of the gradient EM algorithm with current estimate  $(\theta_1,\ldots,\theta_k)$. Also, assume that the next iterate with these current estimates is given by $(\thplusone,\ldots,\thplusk)$. Unrolling the iterate, we have
\begin{align} \label{eqn:em_onestep}
	\theta^+_j = \theta_j - \frac{2\gamma}{n'} \sum_{i=1}^{n'} p_{\thetaall}(x_i,y_i;\theta_j)\left( x_i x_i^T \theta_j - y_i x_i \right).
\end{align}
for all $j\in [k]$.
% \begin{assumption}
	% \label{asm:initial}
	%     We assume the following initialization condition
	%     \begin{align*}
		%         \|\theta^{(0)}_j - \theta^*_j\| \leq c_{\mathsf{ini}} \|\theta^*_j\|
		%     \end{align*}
	%  for all $j \in [k]$, where $c_{\mathsf{ini}}$ is a small constant.
	% \end{assumption}
Furthermore, we assume $|y_i| \leq b$ for all $i \in [n']$ for a non-negative $b$. With this, we are now ready to present the main result of this section.
\begin{theorem}[Gradient EM]
\label{thm:em}
Suppose that  $x_i \stackrel{i.i.d}{\sim}  \cN(0,I_d)$ and that $n' \geq C \,\, \frac{d \log (1/\pi_{\min})}{\pi_{\min}^3}$. Moreover,
	\begin{align*}
		\|\theta_j - \theta^*_j\| \leq c_{\mathsf{ini}} \|\theta^*_j\|
	\end{align*}
	for all $j \in [k]$, where $c_{\mathsf{ini}}$ is a small positive constant (initialization parameter) satisfying  $c_{\mathsf{ini}} <c_2 \frac{\lambda}{\sqrt{\log(1/\pi_{\min}})  \|\theta^*_1\|}$. Then running one iteration of gradient EM algorithm with step size $\gamma$ yields $\{\theta^+_j\}_{j=1}^k$ satisfying
	\begin{align*}
		\|\theta^+_j - \theta^*_j\| \leq \rho \| \theta_j - \theta^*_j\| + \varepsilon,
	\end{align*}
	with probability at least  $1-C_1 \exp(-c_1 \pi_{\min}^4 n') - C_2\exp(-c_2 d)- n'/\mathsf{poly}(d) -n' C_3\exp( - \frac{\lambda^2}{c_{\mathsf{ini}^2} \|\theta^*_1\|^2})$, where 
	\begin{align*}
	\varepsilon &\leq C\gamma \lambda \sqrt{d \log d \log(1/\pi_{\min})} \\
  & + C_1 \gamma \eta' (b +  \sqrt{d\log d \log(1/\pi_{\min})} )^2(c_{\mathsf{ini}} + 1)) \|\theta^*_1\|,
\end{align*} 
	$\rho = (1-2\gamma c (1-\eta) \pi_{\min}^3)$, $\eta' = e^{-((\Delta - C\lambda)^2 - C_2\lambda^2)}$ and $\eta = \left( \frac{1 - e^{-C_2\lambda^2} + (k-1) e^{-(\Delta - C\lambda)^2}}{1 + (k-1) e^{-(\Delta - C\lambda)^2}} \right)$, with $C,C_1,..,c,c_1,..$ as universal positive constants.
\end{theorem}
We defer the proof of the theorem in Appendix~\ref{app:em}. The remarks we made after the AM algorithm continues to hold here as well. 
\begin{remark}[Error floor $\varepsilon$]
Observe that the error floor $\varepsilon$ depends linearly on the step size $\gamma$. The error floor also decays linearly with the misspecification parameter $\lambda$ and an exponentially decaying term dependent on the gap.
\end{remark}
\textbf{Discussion and Comparison between gradient EM and AM:} Note that both the algorithms require initialization and provides exponential convergence with error floor. However, gradient AM minimizes an intuitive min-loss while gradient EM minimizes \emph{optimal} (maximum likelihood in the generative setup) soft-min loss. Moreover, the gradient AM algorithm requires the separation $\Delta  = \Omega(\lambda + \sqrt{\log k}(1+ c_{\mathsf{ini}}))$ (exact condition in Theorem~\ref{thm:am}), whereas we do not have any such requirement for gradient EM. On the flip side, the convergence of gradient EM requires a condition on the initialization parameter $c_{\mathsf{ini}}$ that depends on misspecification $\lambda$, whereas for gradient AM algorithm, no such restriction is imposed. 
\section{Proof Sketches}
\label{sec:proof_sketch}
In this section, we present a rough sketch of the proof of Theorems~\ref{thm:am} and \ref{thm:em}.
\subsection{Gradient AM (Theorem~\ref{thm:am})}
For gradient AM algorithm, based on the current iterates $\{\theta_j\}_{j=1}^k$, we first construct the index sets $\{I_j\}_{j=1}^k$ using $n'$ fresh samples, where $I_j$ consists of all such indices such that $\theta_j$ is a better predictor compared to the other parameters. Similarly, one can construct $\{I^*_j\}_{j=1}^k$ based on $\{\theta^*_j\}_{j=1}^k$. Unrolling gradient AM update (Eq.~\ref{eqn:am_update}), using another set of $n'$ samples we have
\begin{align*}
		\|\thplusone - \theta^*_1\| = \| \theta_1 - \theta^*_1 -  \frac{2\gamma}{n'} \sum_{i\in I_1 }  \left( x_i x_i^T \theta_1 - y_i x_i \right)\|.
\end{align*}
Similar to the gradient EM setup, it turns out that we need to lower bound $\sigma_{\min}(\frac{1}{n'}\sum_{i \in I_j} x_i x_i^T)$. Note that since we use $n'$ fresh samples to construct $I_j$, the set can be considered fixed with respect to the samples used in the gradient step and we can leverage Lemma~\ref{lem:restricted}. We use $\sigma_{\min}(\frac{1}{n'}\sum_{i \in I_1} x_i x_i^T) \geq \sigma_{\min}(\frac{1}{n'}\sum_{i \in I_1 \cap I^*_1} x_i x_i^T)$. Thanks to the suitable initialization and Lemma~\ref{lem:ini_hard}, we show that $|I_1 \cap I^*_1| $ is big enough, yielding a singular value lower bound of $\approx \pi_{\min}^3$. The control of other terms are done similar to the gradient EM setup, and upon combining, we get the final theorem.
\subsection{Gradient EM (Theorem~\ref{thm:em})}
Recall that we consider one iteration of Algorithm~\ref{alg:grad-em} with current and next iterates as $\{\theta_j\}_{j=1}^k$ and $\{\theta^+_j\}_{j=1}^k$ respectively. Recall the update given by Eq.~\ref{eqn:em_onestep}. Without loss of generality, we focus on $j=1$ and use shorthand $p(\theta_1)$ to denote $p_{\thetaall}(x_i,y_i;\theta_1)$. With this we have
\begin{align*}
		\|\thplusone - \theta^*_1\| = \| \theta_1 - \theta^*_1 -  \frac{2\gamma}{n'} \sum_{i=1}^{n'} p(\theta_1) \left( x_i x_i^T \theta_1 - y_i x_i \right)\|.
\end{align*}
We now break the sum to indices $i: (x_i,y_i) \in S^*_1$ and otherwise. When we look at indices such that $(x_i,y_i) \in S^*_1$, after a few algebraic manipulation, it turns out we need to lower bound $\sigma_{\min}[\frac{1}{n'}\sum_{i:(x_i,y_i) \in S^*_1}  x_i x_i^T ]$. Since $\Pr(x_i:(x_i,y_i) \in S^*_1) \geq \pi_{\min}$ by definition, leveraging properties of restricted Gaussians (Lemma~\ref{lem:restricted}), we obtain $\sigma_{\min}[\frac{1}{n'}\sum_{i:(x_i,y_i) \in S^*_1} (1-\eta) x_i x_i^T ] \geq (1-\eta)\pi_{\min}^3$. Furthermore, leveraging the fact that if $(x_i,y_i) \in S^*_1$, we have $p(\theta^*_1) \geq 1-\eta$ (Lemma~\ref{lem:eta_closeness}), and using the norm upper bound on restricted Gaussians (Lemma~\ref{lem:norm}) we control such indices. Finally, combining all the terms and using the geometric parameters succinctly, we obtain the desired result.

\section{Generalization Guarantees}
\label{sec:gen_guarantees}
In this section, we obtain generalization guarantees for the \emph{soft-min} loss functions. Note that similar generalization guarantee for the \emph{min} loss function has appeared in \cite{pal2022learning}. %Here, we keep the problem setup same as \cite{pal2022learning}. Our argument follows the same steps as in \cite{pal2022learning}.

We learn a mixture of functions from $\cX \rightarrow \cY$ for $\cX \subseteq \mathbb{R}^{d}$ fitting data distribution $\cD$ over $(\cX, \cY)$. A learner has access to samples $\{x_i, y_i\}_{i=1}^n$. There is a base class $\cH: \cX \rightarrow \cY$. Here, we work with the setup of list decoding where the learner outputs a list while testing. In \cite{pal2022learning} the list decodable function class has been defined. We rewrite here for completeness.

\begin{definition}
Let $\cH$ be the base function class $\cH$. We  construct a vector valued $k$-list-decodable  function class, namely $\bar{\cH}_k$ such that any $\bar{h} \in \bar{\cH}_k$ is defined as \[
\bar{h} = (h_1(\cdot), \cdots, h_k(\cdot))
\]
such that $h_j \in \cH_j$ for all $j \in [k]$. Thus $\bar{h}$'s map $\cX \rightarrow \cY^{k}$ and form the new function class $\bar{\cH}_k$. 
\end{definition}
To ease notation, we omit the $k$ in $\bar{\cH}$ when clear from context. %Now, we need to define a loss measure. In order to formally define the min-loss setting, let $\ell: \cY \times \cY \rightarrow \reals^+$ be a base loss function. 

In our setting,  the base function class is linear, i.e., for all $j \in [k]$
\begin{align*}
    \cH_j = \cH = \{\inner{\theta}{\cdot}: \forall \theta \in \reals^{d} \text{ s.t } \norm{\theta}_2 \leq R\},
\end{align*}
and the base loss function $\ell: \cY \times \cY \rightarrow \reals^+$ is given by
\begin{align*}
    \ell(h_j(x),y)) = (y - \inner{x}{\theta_j})^2.
\end{align*}
In what follows, we obtain generalization guarantees for bounded covariates and response, i.e., $|y| \leq 1$ and $\|x\| \leq 1$. 
\begin{claim}
\label{claim:bounded}
    For bounded regression problem, the loss function $\ell(h_j(x),y))$ is Lipschitz with parameter $2(1+R)$ with respect to the first argument.
\end{claim}
The proof is deferred to Appendix~\ref{app:gen}. We are interested in the soft loss function, which is a function of the $k$-base loss functions: %. Concretely, we are interested in 
\begin{align*}
    \cL(\Bar{h}(x),y) & = \cL(x,y;\theta_1,\ldots,\theta_k) \\
    &= \sum_{j=1}^k p_{\theta_1,..,\theta_k}(x,y;\theta_j) \left[ y - \langle x,\theta_j \rangle \right]^2 \\
    &= \sum_{j=1}^k p_{\theta_1,..,\theta_k}(x,y;\theta_j) \ell(h_j(x),y),
\end{align*}
where
\begin{align*}
    p_{\theta_1,..,\theta_k}(x,y;\theta_j) = \frac{e^{-(y - \langle x,\theta_j \rangle)^2}}{\sum_{\ell=1}^k e^{-(y - \langle x,\theta_\ell \rangle)^2}}.
\end{align*}

We have $n$ datapoints $\{x_i,y_i\}_{i=1}^n$ drawn from $\cD$ and we want to understand how well this \emph{soft-min} loss generalizes. In order to do that, a standard metric one studies in statistical learning theory is (emprirical) Rademacher Complexity  (\cite{mohri2018foundations}). In our setup, the loss class is defined by
\begin{align*} 
 \lbrace (x,y) \mapsto   \sum_{j=1}^k p_{\theta_1,..,\theta_k}(x,y;\theta_j) \ell(h_j(x),y); \{\theta_j:\|\theta_j\|\leq R \}_{j=1}^k  \rbrace.
\end{align*}
Let us define this class as $\Phi$. The Rademacher complexity of the loss class is given by
\small
\begin{align*}
    & \rad_n (\Phi)  =  \EE_{\bsigma} \left[ \sup_{\bar{h} \in \bar{\cH}_k} \bigg | \frac{1}{n} \sum_{i=1}^{n} \sigma_i \cL( \bar{h}(x_i),y_i) \bigg | \right] \\
    &= \EE_{\bsigma} \left[ \sup_{\{\theta_j:\|\theta_j\|\leq R \}_{j=1}^k} \bigg | \frac{1}{n} \sum_{i=1}^{n} \sigma_i \sum_{j=1}^k p_{\theta_1,..,\theta_k}(x,y;\theta_j) \ell(h_j(x),y) \bigg | \right],
\end{align*}
\normalsize
where $\bsigma$ is a set of Rademacher RV's $\{\sigma_i\}_{i=1}^n$. We have the following result:
\begin{lemma}
\label{lem:rad}
    The Rademacher complexity of $\Phi$ satisfies
    \begin{align*}
         \rad(\Phi) \leq 4k(1+R) \rad(\mathcal{H}) \leq \frac{4kR(1+R)}{\sqrt{n}}.
    \end{align*}
\end{lemma}
We observe that the (empirical) Rademacher complexity of the \emph{soft-min} loss class does not blow-up provided the complexity of the base class $\mathcal{H}$ is controlled. Moreover, since the base class is a linear hypothesis class (with bounded $\ell_2$ norm), the Rademacher complexity scales as $\mathcal{O}(1/\sqrt{n})$, resulting in the above bound. The proof is deferred in Appendix~\ref{app:gen}. In a nutshell, we consider a bigger class of all possible convex combination of the base losses, and connect $\Phi$ to that bigger function class.

\section{Conclusion and Open Problems}
\label{sec:con}
In this work, we have studied the agnostic setup for mixed linear regression, and show that EM and AM algorithms are strong enough to provide provable guarantees even in this setup. However we believe such algorithms may be used in a broader context of agnostic learning. We conclude the paper with a few interesting problems. Beyond mixture of linear regressions, can this agnostic setup be used for other problems such as mixture of classifiers, mixture of experts, to name a few? What is the role of Gaussian covariates in such an agnostic setting? Can we relax this to some extent? In \cite{ghosh2019max} it is explained how restricted Gaussian analysis can be extended to sub-Gaussians satisfying a \emph{small ball} condition for the particular problem of max-affine regression. Another interesting direction is to analyze the AM based algorithms without resampling in the agnostic setup, leveraging techniques like Leave One Out (LOO) as an example. We keep these as our future endevors.
\clearpage

\section*{Impact Statement}
This paper presents work whose goal is to advance the field of Machine Learning. There are many potential societal consequences of our work, none which we feel must be specifically highlighted here.

\textbf{Acknowledgements.} This research is supported in part by NSF awards 2133484, 2217058, 2112665.

\bibliography{references}
\bibliographystyle{icml2024}

\newpage
\appendix
\onecolumn
\section{Proof of Theorem~\ref{thm:am}}
\label{app:am}
Without loss of generality, let us focus on  $\thplusone$.
	We have
	\begin{align*}
		\| \thplusone -\theta^*_1 \| &= \| \theta_1 - \theta^*_1 -\frac{\gamma}{n'}\sum_{i \in I_1}\gradone \| \\
		& = \| (\theta_1 - \theta^*_1) - \frac{\gamma}{n'}\sum_{i \in I_1}(\gradone -\nabla F_i(\theta^*_1)) - \frac{\gamma}{n'}\sum_{i \in I_1}\nabla F_i(\theta^*_1) \| \\
		& \leq \underbrace{ \| (\theta_1 - \theta^*_1) - \frac{\gamma}{n'}\sum_{i \in I_1}(\gradone -\nabla F_i(\theta^*_1))\|}_{T_1} + \frac{\gamma}{n'} \underbrace{\|\sum_{i \in I_1}\nabla F_i(\theta^*_1) \|}_{T_2}.
	\end{align*}
	
	Let us first consider $T_1$. Substituting the gradients, we obtain
	\begin{align*}
		T_1 = \| (I - \frac{2\gamma}{n}\sum_{i \in I_1} x_i x_i^\top) (\theta_1 - \theta^*_1)\| = \| (I - \frac{2\gamma}{n'}\sum_{i: (x_i,y_i) \in S_1} x_i x_i^\top) (\theta_1 - \theta^*_1)\|.
	\end{align*}
	We require a lower bound on
	\begin{align*}
		\sigma_{\min}(\frac{1}{n}\sum_{i \in I_1} x_i x_i^\top ) \geq  \sigma_{\min}(\frac{1}{n'}\sum_{i: (x_i,y_i) \in S_1 \cap S^*_1} x_i x_i^\top )
	\end{align*}
	Similar to the EM framework, in order to bound the above, we need to look at the behavior of the covariates (which are standard Gaussian) over the restricted set given by $ S_1 \cap S^*_1$. Note that since we are resampling at each step, and using fresh set of samples to construct $S_j$ and another fresh set of samples to run the Gradient AM algorithm, we can directly use Lemma~\ref{lem:restricted} here. Moreover, we  use the fact that $|i: (x_i,y_i) \in S_1 \cap S^*_1| \geq C |i: (x_i,y_i) \in S^*_1|  \geq C' \pi_{\min} n$ with probability at least $1-C\exp(-\pi_{\min} n$) where we use  the initialization Lemma~\ref{lem:ini_hard}. Thus, we have
	\begin{align*}
		\sigma_{\min}(\frac{1}{n'}\sum_{i:(x_i,y_i) \in S_1} x_i x_i^\top ) \geq c \pi_{\min}^3
	\end{align*}
	with probability at least $1-C_1 \exp(-C_2 \pi_{\min}^4 n') - C_3 \exp(-\pi_{\min} n')$ provided $n' \geq C \frac{d \log(1/\pi_{\min})}{\pi_{\min}^3}$. As a result,
	\begin{align*}
		T_1 \leq (1- c \gamma \pi_{\min}^3) \|\theta_1 -\theta^*_1\|,
	\end{align*}
	with probability at least $1-C_1 \exp(-C_2 \pi_{\min}^4 n') $.
	\vspace{2mm}
	
	\noindent Let us now consider the term $T_2$. We have
	\begin{align*}
		T_2 &= \frac{\gamma}{n} \|\sum_{i:(x_i,y_i) \in S_1}\nabla F_i(\theta^*_1) \|  \\
		& \leq \frac{\gamma}{n} \sum_{i:(x_i,y_i) \in S_1} \|\nabla F_i(\theta^*_1) \| \\
		& = \frac{\gamma}{n}   \sum_{i:(x_i,y_i) \in S_1 \cap S^*_1} \|\nabla F_i(\theta^*_1) \| + \frac{\gamma}{n} \sum_{j=2}^{k} \sum_{i:(x_i,y_i) \in S_1 \cap S^*_j} \|\nabla F_i(\theta^*_1) \|
	\end{align*}
	When $\{i:(x_i,y_i) \in S^*_1\}$, we have
	\begin{align*}
		\|\nabla F_i (\theta^*_1)\| &= 2|y_i - \langle x_i,\theta^*_1 \rangle| \|x_i\| \\
		& \leq 2 \lambda \|x_i\| \leq C \lambda \sqrt{d\log d \log(1/\pi_{\min})} 
	\end{align*}
	with probability at least $1-n'/\mathsf{poly}(d)$, where in the first inequality, we have used the misspecification assumption, and in the second inequality, we use Lemma~\ref{lem:norm}. Let us now compute an upper bound on $\|\nabla F_i(\theta^*_1)\|$, which we use to bound the second part. We have
 \begin{align*}
     \|\nabla F_i(\theta^*_1)\| &\leq \|x_i\|^2 \|\theta^*_1\| + \|x_i\| |y_i| \\
     & \leq C_1 d\log d \log(1/\pi_{\min}) \|\theta^*_1\| + C b \sqrt{d\log d \log(1/\pi_{\min})}
 \end{align*}
with probability at least $1-1/\mathsf{poly}(d)$. 

	With this, we have
	\begin{align*}
		T_2 &\leq \frac{\gamma}{n} |I_1 \cap I^*_1| C \lambda \sqrt{d\log d \log(1/\pi_{\min})} + \frac{\gamma}{n} \sum_{j=2}^k |I_1 \cap I^*_j| \bigg (  C_1 d\log d \log(1/\pi_{\min}) \|\theta^*_1\| \\
    & \qquad \qquad + C b \sqrt{d\log d \log(1/\pi_{\min})} \bigg) \\
		& \leq\gamma  C \lambda \sqrt{d\log d \log(1/\pi_{\min})} + C_1 \gamma (k-1) P_e \left[ d\log d \log(1/\pi_{\min}) \|\theta^*_1\| + C b \sqrt{d\log d \log(1/\pi_{\min})}\right],
	\end{align*}
 with probability at least $1-\exp(-cP_e n) - \frac{n'}{\mathsf{poly}(d)}-\frac{P_e n}{\mathsf{poly}(d)}$, where $P_e$ is defined in Lemma~\ref{lem:ini_hard}. In this case, we use $|I_1 \cap I^*_1| \leq n'$ (trivially holds) as well as the standard binomial concentration on $|I_1 \cap I^*_j|$ with mean at most $n' P_e$ with probability at least $1-\exp(-cP_e n')$. Moreover we take the union bound. Here, we use Lemma~\ref{lem:norm} along with the fact that $|y_i| \leq b$.
 % Finally, we have
	% \begin{align*}
	% 	T_2 &\leq \gamma  C \lambda\, \sqrt{d\log d \log(1/\pi_{\min})} + c_1 \gamma (k-1) P_e ( \max_i \|x_i\|^2 \|\theta^*_1\| + \max_i \|x_i\| |y_i|) \\
	% 	& \leq\gamma  C \lambda \sqrt{d\log d \log(1/\pi_{\min})} + C_1 \gamma \sqrt{d\log d \log(1/\pi_{\min})}^2 (k-1)P_e \|\theta^*_1\|,
	% \end{align*}

	Combining $T_1$ and $T_2$, we have
	\begin{align*}
		\|\thplusone - \theta^*_1\| & \leq (1- c \gamma \pi_{\min}^3) \|\theta_1 -\theta^*_1\|+ C \gamma \lambda \sqrt{d\log d \log(1/\pi_{\min})} \\
		& +C_1 \gamma (k-1) P_e \left[ d\log d \log(1/\pi_{\min}) \|\theta^*_1\| + C b \sqrt{d\log d \log(1/\pi_{\min})}\right],
	\end{align*}
	with probability at least $1-C_1 \exp(-C_2 \pi_{\min}^4 n') - \exp(-cP_e n') - \frac{n'}{\mathsf{poly}(d)}$ .

\subsection{Good Initialization}
We stick to analyzing $\thplusone$. In the following lemma, we only consider $\theta_2$. In general, the same argument holds for $\{\theta_3, \ldots, \theta_k\}$.
\begin{lemma}
\label{lem:ini_hard}
We have
\begin{align*}
    P_e  & = \PP \bigg ( F_i (\theta_1) > F_i(\theta_2) | i \in I^*_1 \bigg ) \\
    & \leq  4 \exp \bigg( - \frac{1}{c_{\mathsf{ini}^2} \max_{j \in [k]}\|\theta^*_j\|^2}  \bigg[ \frac{\Delta -\lambda}{2}  \bigg]^2 \bigg )
\end{align*}
\end{lemma}
\noindent Let us consider the event 
\begin{align*}
	F_i(\theta_1) > F_i(\theta_2),
\end{align*}
which is equivalent to 
\begin{align*}
|y_i - \langle x_i,\theta_1 \rangle | > |y_i - \langle x_i,\theta_2 \rangle |.
\end{align*}
Let us look at the left hand side of the above inequality. We have
\begin{align*}
	 & |y_i - \langle x_i, \theta^*_1 \rangle + \langle x_i , \theta_1 - \theta^*_1 \rangle | \\
	& \leq |y_i - \langle x_i, \theta^*_1 \rangle | + | \langle x_i , \theta_1 - \theta^*_1 \rangle |\\
	& \leq \lambda + | \langle x_i , \theta_1 - \theta^*_1 \rangle |,
\end{align*}
where we have used the fact that if $i \in I^*_1$, the first term is at most $\lambda$. 

Similarly, for the right hand side, we have
\begin{align*}
	& |y_i - \langle x_i,\theta^*_2 \rangle - \langle x_i, \theta_2 - \theta^*_2 \rangle | \\
	& \geq |y_i - \langle x_i,\theta^*_2 \rangle| - | \langle x_i, \theta_2 - \theta^*_2 \rangle |  \\
	& \geq \Delta -  | \langle x_i, \theta_2 - \theta^*_2 \rangle | 
\end{align*}
where we use the fact that if $i \in I^*_1$, the first term is lower bounded by $\Delta$.

Combining these, we have
\begin{align*}
	 \PP \bigg ( F_i (\theta_1) > F_i(\theta_2) | i \in I^*_1 \bigg )& \leq \PP \bigg(  | \langle x_i , \theta_1 - \theta^*_1 \rangle | + | \langle x_i, \theta_2 - \theta^*_2 \rangle |  \geq \Delta - \lambda \bigg ) \\
	 & \leq \PP \bigg(  | \langle x_i , \theta_1 - \theta^*_1 \rangle |  \geq  \frac{\Delta -\lambda}{2} \bigg ) +  \PP \bigg(  | \langle x_i , \theta_2 - \theta^*_2 \rangle |  \geq  \frac{\Delta -\lambda}{2} \bigg ) 
\end{align*}
Let us look at the first term. Lemma~\ref{lem:restricted} shows that if $i \in I^*_1$ (accordingly $(x_i,y_i) \in S^*_1$), the distribution of $x_i - \mu_\tau$ is subGaussian with (squared) parameter at most $C(1+\log(1/\pi_{\min}))$, where $\mu_\tau$ is the mean of $x_i$ (under the restriction $(x_i,y_i) \in S^*_1$). With this we have
\begin{align*}
 \PP \bigg(  | \langle x_i , \theta_1 - \theta^*_1 \rangle |  \geq  \frac{\Delta -\lambda}{2} \bigg ) &\leq  \PP \bigg(  | \langle x_i - \mu_\tau , \theta_1 - \theta^*_1 \rangle | + \|\mu_\tau\| \|\theta_1 - \theta^*_1\| \geq  \frac{\Delta -\lambda}{2} \bigg ) \\
 &\leq \PP \bigg(  | \langle x_i - \mu_\tau , \theta_1 - \theta^*_1 \rangle | \geq  \frac{\Delta -\lambda}{2} -  c_{\mathsf{ini}} C \sqrt{\log(1/\pi_{\min}})  \|\theta^*_1\| \bigg ) 
\end{align*}
where we use the initialization condition $\|\theta_1 - \theta^*_1\| \leq c_{\mathsf{ini}}  \|\theta^*_1\|$, and from Lemma~\ref{lem:restricted}, we have $\|\mu_\tau\|^2 \leq C \log(1/\pi_{\min})$. 

Now, provided $\Delta - \lambda > C ( c_{\mathsf{ini}} \sqrt{\log(1/\pi_{\min}})  \|\theta^*_1\| ) + C_1 \sqrt{1+ \log(1/\pi_{\min})}$, using sub-Gaussian concentration, we obtain
\begin{align*}
	 \PP \bigg(  | \langle x_i , \theta_1 - \theta^*_1 \rangle |  \geq  \frac{\Delta -\lambda}{2} \bigg ) \leq 2 \exp \bigg( - \frac{1}{c_{\mathsf{ini}^2} \|\theta^*_1\|^2}  \bigg[ \frac{\Delta -\lambda}{2} \bigg]^2 \bigg ).
\end{align*}
Similarly, for the second term, similar calculation yields
\begin{align*}
	\PP \bigg(  | \langle x_i , \theta_2 - \theta^*_2 \rangle |  \geq  \frac{\Delta -\lambda}{2} \bigg ) \leq 2 \exp \bigg( - \frac{1}{c_{\mathsf{ini}^2} \|\theta^*_2\|^2}  \bigg[ \frac{\Delta -\lambda}{2}  \bigg]^2 \bigg ),
\end{align*}
and hence
\begin{align*}
 \PP \bigg ( F_i (\theta_1) > F_i(\theta_2) | i \in I^*_1 \bigg ) \leq 4 \exp \bigg( - \frac{1}{c_{\mathsf{ini}^2} \max_{j \in [k]}\|\theta^*_j\|^2}  \bigg[ \frac{\Delta -\lambda}{2}  \bigg]^2 \bigg )
\end{align*}
which proves the lemma.

\section{Proof of Theorem~\ref{thm:em}}
\label{app:em}
Let us look at the iterate of gradient EM after one step and without loss of generality, we focus on recovering $\theta^*_1$. We have
	\begin{align*}
		\|\thplusone - \theta^*_1\| = \| \theta_1 - \theta^*_1 -  \frac{2\gamma}{n'} \sum_{i=1}^{n'} p_{\thetaall}(x_i,y_i;\theta_1)\left( x_i x_i^T \theta_1 - y_i x_i \right)\|
	\end{align*}
	Let us use the shorthand $p(\theta_1)$ to denote $p_{\thetaall}(x_i,y_i;\theta_1)$ and $p(\theta^*_1)$ to denote $p_{\theta^*_1,\ldots,\theta^*_k}(x_i,y_i;\theta^*_1)$ respectively. We have
	\begin{align*}
		\|\thplusone - \theta^*_1\| &= \| \theta_1 - \theta^*_1 -  \frac{2\gamma}{n'} \sum_{i:(x_i,y_i)\in S^*_1} p(\theta_1)\left( x_i x_i^T \theta_1 - y_i x_i \right) -\frac{2\gamma}{n'} \sum_{i:(x_i,y_i)\notin S^*_1} p(\theta_1)\left( x_i x_i^T \theta_1 - y_i x_i  \right)\| \\
		& \leq \underbrace{\|\theta_1 - \theta^*_1 -  \frac{2\gamma}{n'} \sum_{i: (x_i,y_i) \in S^*_1} p(\theta_1)\left( x_i x_i^T \theta_1 - y_i x_i  \right) -\frac{2\gamma}{n'} \sum_{i: (x_i,y_i) \notin S^*_1} p(\theta_1)\left( x_i x_i^T \theta_1 - y_i x_i  \right)\|}_{T_1} 
  % \\
		% & \quad \quad + \underbrace{\frac{2\gamma}{n}\|\sum_{i=1}^n [p(\theta_1) - p(\theta^*_1)] (x_i x_i^T \theta_1 - y_i x_i) \|}_{T_2}.
	\end{align*}
	
	First we argue from the separability and the closeness condition that, if $(x_i,y_i) \in S^*_1$, the probability $p(\theta_1)$ is bounded away from $0$. Lemma~\ref{lem:eta_closeness} shows that conditioned on $(x_i,y_i) \in S^*_j$, we have $p_{\theta_1,\ldots,\theta_k}(x_i,y_i;\theta_j) \geq 1-\eta$, where 
		\begin{align*}
			\eta = \left( \frac{1 - e^{-C_2\lambda^2} + (k-1) e^{-(\Delta - C\lambda)^2}}{1 + (k-1) e^{-(\Delta - C\lambda)^2}} \right).
		\end{align*}
	
	with probability at least $1- C_3 \exp \bigg( - C_1 \frac{\lambda^2}{c_{\mathsf{ini}^2} \|\theta^*_1\|^2}   \bigg )$. With this, let us look at $T_1$. We have
	\begin{align*}
		T_1 \leq \underbrace{\|\theta_1 - \theta^*_1 -  \frac{2\gamma}{n'} \sum_{i: (x_i,y_i) \in S^*_1} p(\theta_1)\left( x_i x_i^T \theta_1 - y_i x_i  \right)\|}_{T_{11}} + \underbrace{\frac{2\gamma}{n'}\| \sum_{i: (x_i,y_i) \notin S^*_1} p(\theta_1)\left( x_i x_i^T - y_i x_i \right)\|}_{T_{12}}.
	\end{align*}
	We continue to upper bound $T_{11}$:
	\begin{align*}
		T_{11} &\leq \|\theta_1 - \theta^*_1 -  \frac{2\gamma}{n'} \sum_{i: (x_i,y_i) \in S^*_1} p(\theta_1)\left( x_i x_i^T \theta_1 - y_i x_i \right)\| \\
		& \leq \|\theta_1 - \theta^*_1 -  \frac{2\gamma}{n'} \sum_{i: (x_i,y_i)\in S^*_1} p(\theta_1)\left( x_i x_i^T \theta_1 - x_i x_i^T \theta^*_1 \right)\| + \frac{2\gamma}{n'} \|\sum_{i: (x_i,y_i)\in S^*_1} p(\theta_1)\left( x_i x_i^T \theta^*_1 - y_i x_i\right)\| \\
		& \leq \| \bigg [ I - \frac{2\gamma}{n'}\sum_{i: (x_i,y_i) \in S^*_1}p(\theta_1) x_i x_i^T \bigg ] (\theta_1 - \theta^*_1 ) \| + \frac{2\gamma}{n'} \sum_{i: (x_i,y_i) \in S^*_1} p(\theta_1) | y_i - \langle x_i, \theta^*_1 \rangle | \|x_i\| \\
		& \leq \| \bigg [ I - \frac{2\gamma}{n'}\sum_{i: (x_i,y_i) \in S^*_1}p(\theta_1) x_i x_i^T \bigg ] (\theta_1 - \theta^*_1 ) \| + C \lambda \gamma \,\, \sqrt{d \log d \log(1/\pi_{\min})},
	\end{align*}
	with probability at least $1- C_3 n'\exp \bigg( - C_1 \frac{\lambda^2}{c_{\mathsf{ini}^2} \|\theta^*_1\|^2}   \bigg ) - n'/\mathsf{poly}(d)$, where we use the misspecification condition, $|y_i - \langle x_i, \theta^*_1 \rangle| \leq \lambda$ for all $(x_i,y_i) \in S^*_1$, along with the fact that the number of such indices is trivially upper bounded by the total number of observations, $n$. Moreover, we also use Lemma~\ref{lem:norm} to bound $\|x_i\|$. 
	
   Note that since $(x_i,y_i) \in S^*_1$, we have $p(\theta_1) \geq 1-\eta$. We need to look at $\sigma_{\min}\left(\frac{1}{n'}\sum_{i: (x_i,y_i)\in S^*_1} p(\theta_1) x_i x_i^T \right)$, where $p(\theta_1) \geq 1-\eta$. We use the fact that
	\begin{align*}
		\sigma_{\min}\left(\frac{1}{n'}\sum_{i:(x_i,y_i)\in S^*_1} p(\theta_1) x_i x_i^T \right) \geq \sigma_{\min}\left(\frac{1}{n'}\sum_{i:(x_i,y_i)\in S^*_1} (1-\eta) x_i x_i^T \right).
	\end{align*}
	
	Note that we need to analyze the behavior of the data restricted on the set $S^*_1$. In particular we are interested in the second moment estimation of such restricted Gaussian random variable. We show that, conditioned on $S^*_1$, the distribution of $x_i$ changes to a sub-Gaussian with a shifted mean. Lemma~\ref{lem:restricted} characterizes the behavior as well as the second moment estimation for such variables.
	
	We invoke the Lemma~\ref{lem:restricted} and use the standard binomial concentration to obtain $|i: (x_i,y_i) \in S^*_1| \geq C \pi_{\min} n$ with probability at least $1- \exp(-c \pi_{\min} n)$. With this, we obtain
	\begin{align*}
		\sigma_{\min}\left(\frac{1}{n'}\sum_{i:(x_i,y_i) \in S^*_1} (1-\eta) x_i x_i^T \right) \geq c (1-\eta) \pi_{\min}^3
	\end{align*}
	with probability at least $1- C_1 \exp(-C_2 \pi_{\min}^4 n')$, provided $n' \geq C \frac{d \log (1/\pi_{\min})}{\pi_{\min}^3}$.

	Using this, we obtain
	\begin{align*}
		T_{11} \leq (1-2\gamma c (1-\eta) \pi_{\min}^3) \| \theta_1 - \theta^*_1\| + C\gamma \lambda \sqrt{d \log d \log(1/\pi_{\min})}.
	\end{align*}
	with high probability. Let us now look at $T_{12}$. We have
	\begin{align*}
		T_{12} & = \frac{2\gamma}{n'}\| \sum_{i : (x_i,y_i)\notin S^*_1} p(\theta_1)\left( x_i x_i^T \theta_1 - y_i x_i \right)\| \\
		& \leq \frac{2\gamma}{n'} \sum_{i : (x_i,y_i) \notin S^*_1} p(\theta_1) \| x_i x_i^T \theta_1 - y_i x_i \| \\
		& \stackrel{(i)}{\leq} \frac{2\gamma \eta'}{n'} \sum_{i : (x_i,y_i) \notin S^*_1} | y_i - x_i^T \theta_1| \|x_i \| \\
		& \leq \frac{2\gamma \eta'}{n'} \sum_{i : (x_i,y_i) \notin S^*_1} (|y_i| + \|x_i\| \|\theta_1\|) \|x_i\| \\
		& \stackrel{(ii)}{\leq }\frac{2\gamma \eta'}{n'} \sum_{i  : (x_i,y_i) \notin S^*_1} (b + C \sqrt{d\log d \log(1/\pi_{\min})}) [\|\theta_1 - \theta^*_1\| + \|\theta^*_1 \|]) \sqrt{d\log d \log(1/\pi_{\min})} \\
		& \leq 2 \gamma \eta' (b + C \sqrt{d\log d \log(1/\pi_{\min})})^2(c_{\mathsf{ini}} + 1)) \|\theta^*_1\|.
	\end{align*}
with probability at least $1-n'/\mathsf{poly}(d) - C_3 n'\exp \bigg( - C_1 \frac{\lambda^2}{c_{\mathsf{ini}^2} \|\theta^*_1\|^2}   \bigg )$ (using union bound). Here $(i)$ follows from the fact that $p(\theta^*_1) \leq \eta'$ where $\eta'= e^{-((\Delta - C\lambda)^2 - C_2\lambda^2)}.$ (since $(x_i,y_i) \notin S^*_1$, which follows from Lemma~\ref{lem:eta_closeness}), $(ii)$ follows from the fact that $|y_i| \leq b$ for all $i$. Moreover, since $\{S^*_j\}_{j=1}^d$ partitions $\mathbb{R}^d$, $(x_i,y_i) \notin S^*_1$ implies that $ (x_i,y_i) \in S^*_\ell$ where $\ell \in [k] \setminus \{1\}$, and we can invoke Lemma~\ref{lem:norm}.

% We now look at $T_2$. We argue that $T_2$ is small since $p(\theta_1)$ is close to $p(\theta^*_1)$. In order to achieve this, we leverage the initialization condition. Lemma~\ref{lem:ini_soft} shows that, \begin{align*}
% 			|p(\theta_1)- p(\theta^*_1)| \leq c_{\mathsf{ini}} \,\,\sqrt{k} \,\, \max_{j \in [k]} \|\theta^*_j\| \bigg[ 2b + C\sqrt{d}(2+c_{\mathsf{ini}})\max_{j \in [k]} \|\theta^*_j\| \bigg].
% 		\end{align*}
% 		with probability at least $1-2\exp(-c_1 d)$.

%      Now, let us use this to bound $T_2$. We have
% 	\begin{align*}
% 		T_2 & =\frac{2\gamma}{n}\|\sum_{i=1}^n [p(\theta_1) - p(\theta^*_1)] (x_i x_i^T \theta_1 - y_i x_i) \| \\
% 		& \leq \frac{2\gamma}{n} \sum_{i=1}^n |p(\theta_1)- p(\theta^*_1)| \|x_i\| | y_i - x_i^T \theta_1| \\
% 		& \leq \frac{2\gamma}{n} \sum_{i=1}^n |p(\theta_1)- p(\theta^*_1)| (|y_i| + \|x_i\| \|\theta_1\|) \\
% 		& \leq \frac{2\gamma}{n} \sum_{i=1}^n |p(\theta_1)- p(\theta^*_1)| (b + C\sqrt{d}(1+ c_{\mathsf{ini}})\|\theta^*_1\|) \\
% 		& \leq 2 \, \gamma c_{\mathsf{ini}} \, \sqrt{k} (b + C\sqrt{d}(1+ c_{\mathsf{ini}})\|\theta^*_1\|) \bigg[ 2b+ C\sqrt{d}(2+c_{\mathsf{ini}})\max_{j \in [k]} \|\theta^*_j\| \bigg].
% 	\end{align*}
% 	with probability at least $1-c\exp(-c_1 d)$, where we have used the chi-squared concentration to bound $\|x_i\|$. 
	
	\textbf{Collecting all the terms:} We now collect the terms and combine them to obtain
	\begin{align*}
		\|\thplusone - \theta^*_1\| &\leq T_{11} + T_{12}  \\
		& \leq (1-2\gamma c (1-\eta) \pi_{\min}^3) \| \theta_1 - \theta^*_1\| + C\gamma \lambda \sqrt{d \log d \log(1/\pi_{\min})} \\
		&+ 2 \gamma \eta' (b + C \sqrt{d\log d \log(1/\pi_{\min})} )^2(c_{\mathsf{ini}} + 1)) \|\theta^*_1\|.
	\end{align*}
	with probability at least $1-C_1 \exp(-c_1 \pi_{\min}^4 n') - C_2\exp(-c_2 d) - n'/\mathsf{poly}(d) - n' \,\,C_3\exp \bigg( - \frac{\lambda^2}{c_{\mathsf{ini}^2} \|\theta^*_1\|^2}   \bigg )$.
	
	Let $\rho = (1-2\gamma c (1-\eta) \pi_{\min}^3)$ and we choose $\gamma$ such that $\rho <1$. We obtain
	\begin{align*}
		\|\thplusone - \theta^*_1\| \leq \rho \| \theta_1 - \theta^*_1\| + \varepsilon,
	\end{align*}
	where 
	\begin{align*}
		\varepsilon &\leq C\gamma \lambda \sqrt{d \log d \log(1/\pi_{\min})} + 2 \gamma \eta' (b + C \sqrt{d\log d \log(1/\pi_{\min})} )^2(c_{\mathsf{ini}} + 1)) \|\theta^*_1\|,
	\end{align*}
	with probability at least $1-C_1 \exp(-c_1 \pi_{\min}^4 n') - C_2\exp(-c_2 d)- n'/\mathsf{poly}(d) -n' C_3\exp \bigg( - \frac{\lambda^2}{c_{\mathsf{ini}^2} \|\theta^*_1\|^2}   \bigg )$.

\subsection{Proofs of Auxiliary Lemmas:}

\begin{lemma}
		\label{lem:eta_closeness}
		For any $(x_i,y_i) \in S^*_j$, we have $p_{\theta_1,\ldots,\theta_k}(x_i,y_i;\theta_j) \geq 1-\eta$, where 
		\begin{align*}
			\eta = \left( \frac{1 - e^{-C_2\lambda^2} + (k-1) e^{-(\Delta - C\lambda)^2}}{1 + (k-1) e^{-(\Delta - C\lambda)^2}} \right).
		\end{align*}
  Moreover, for $(x_i,y_i) \notin S^*_j$ we have
  \begin{align*}
     p_{\theta_1,\ldots,\theta_k}(x_i,y_i;\theta_j) \leq e^{-((\Delta - C\lambda)^2 - C_2\lambda^2)}.
  \end{align*}
\end{lemma}

\begin{proof}
	Consider any $(x_i,y_i) \in S^*_j$ and use the definition of $ p_{\theta_1,\ldots,\theta_k}(x_i,y_i;\theta_j)$. We obtain
	\begin{align*}
		p_{\theta_1,\ldots,\theta_k}(x_i,y_i;\theta_j) &= \frac{e^{-(y_i - \langle x_i,\theta_j \rangle)^2}}{\sum_{\ell=1}^k e^{-(y_i - \langle x_i,\theta_\ell \rangle)^2}}
	\end{align*}

Note that 
\begin{align*}
    |y_i - \langle x_i,\theta_j \rangle| &= |y_i - \langle x_i,\theta^*_j \rangle + \langle x_i, \theta^*_j - \theta_j \rangle | \\
    & \leq | y_i - \langle x_i,\theta^*_j \rangle| + |\langle x_i, \theta^*_j - \theta_j \rangle |
\end{align*}
Furthermore, using reverse triangle inequality, we also have
\begin{align*}
    |y_i - \langle x_i,\theta_j \rangle| \geq | y_i - \langle x_i,\theta^*_j \rangle| - |\langle x_i, \theta^*_j - \theta_j \rangle |.
\end{align*}
Since we are re-sampling at every step, and from the initialization condition, we handle the random variable $\langle x_i, \theta^*_j - \theta_j \rangle$.

Using Lemma~\ref{lem:restricted} shows that if  $(x_i,y_i) \in S^*_1$, the distribution of $x_i - \mu_\tau$ is subGaussian with (squared) parameter at most $C(1+\log(1/\pi_{\min}))$, where $\mu_\tau$ is the mean of $x_i$ (under the restriction $(x_i,y_i) \in S^*_1$). With this we have
\begin{align*}
 \PP \bigg(  | \langle x_i , \theta_1 - \theta^*_1 \rangle |  \geq  C \lambda \bigg ) &\leq  \PP \bigg(  | \langle x_i - \mu_\tau , \theta_1 - \theta^*_1 \rangle | + \|\mu_\tau\| \|\theta_1 - \theta^*_1\| \geq  C \lambda \bigg ) \\
 &\leq \PP \bigg(  | \langle x_i - \mu_\tau , \theta_1 - \theta^*_1 \rangle | \geq  C \lambda -  c_{\mathsf{ini}} C_1 \sqrt{\log(1/\pi_{\min}})  \|\theta^*_1\| \bigg ) 
\end{align*}
where we use the initialization condition $\|\theta_1 - \theta^*_1\| \leq c_{\mathsf{ini}}  \|\theta^*_1\|$, and from Lemma~\ref{lem:restricted}, we have $\|\mu_\tau\|^2 \leq C \log(1/\pi_{\min})$.

Now, provided $c_{\mathsf{ini}} < C_2 \frac{\lambda}{\sqrt{\log(1/\pi_{\min}})  \|\theta^*_1\|}$, using sub-Gaussian concentration, we obtain
\begin{align*}
	 \bigg(  | \langle x_i , \theta_1 - \theta^*_1 \rangle |  \geq  C \lambda \bigg ) \leq 2 \exp \bigg( - C_1 \frac{1}{c_{\mathsf{ini}^2} \|\theta^*_1\|^2}  \lambda^2 \bigg ).
\end{align*}

Using the assumption, i,.e., the separability and the misspecification condition, we obtain
	\begin{align*}
		p_{\theta_1,\ldots,\theta_k}(x_i,y_i;\theta_j) &\geq  \frac{e^{-C_2 \lambda^2}}{e^{-(y_i - \langle x_i,\theta_j \rangle)^2} + \sum_{\ell \neq j} e^{-(y_i - \langle x_i,\theta_\ell \rangle)^2}} \\
		& \geq  \frac{e^{-C_2\lambda^2}}{e^{-(y_i - \langle x_i,\theta_j \rangle)^2} + (k-1) e^{-(\Delta - C\lambda)^2}} \\
		& \geq  \frac{e^{-C_2\lambda^2}}{1 + (k-1) e^{-(\Delta - C\lambda)^2}} \\
		& = 1 - \left( \frac{1 - e^{-C_2\lambda^2} + (k-1) e^{-(\Delta - C\lambda)^2}}{1 + (k-1) e^{-(\Delta - C\lambda)^2}} \right).
	\end{align*}

 Let us look at the condition $(x_i,y_i) \notin S^*_j$. Since $\{S^*_j\}_{j=1}^k$ partitions $\mathbb{R}^d$, $(x_i,y_i) \in S^*_{j'}$ for $j' \in [k]$. With this,

 \begin{align*}
		p_{\theta_1,\ldots,\theta_k}(x_i,y_i;\theta_j) &\leq  \frac{e^{-(\Delta - C\lambda)^2}}{e^{-(y_i - \langle x_i,\theta_{j'} \rangle)^2} + \sum_{\ell \neq j'} e^{-(y_i - \langle x_i,\theta_\ell \rangle)^2}} \\
		& \leq  \frac{e^{-(\Delta - C\lambda)^2}}{e^{-C_2\lambda^2} + 0} =  e^{-((\Delta - C\lambda)^2 - C_2\lambda^2)}.
	\end{align*}

 The above events occur with probability at least $1- C_3 \exp \bigg( - C_1 \frac{\lambda^2}{c_{\mathsf{ini}^2} \|\theta^*_1\|^2}   \bigg )$.

 \end{proof}

\begin{lemma}
	\label{lem:restricted}
	Suppose $x\sim \cN(0,I_d)$ and a fixed set $S$ such that $\PP(x \in S) \geq \nu$. Let $\tau$ denote the restriction of $x$ onto $S$. Moreover, suppose we have $n$ draws from a standard Gaussian and $m$ of them falls in $S$. Provided $n \geq \frac{C \log (1/\nu)}{\nu^3} d$, we have
	\begin{align*}
		\sigma_{\min} \left(  \frac{1}{m} \sum_{i=1}^m \tau_i \tau_i^T \right) \geq \frac{C}{2} \nu^2,
	\end{align*}
	with probability at least $1- 2\exp(-c_1  \nu^4 n)$.
\end{lemma}
\begin{proof}
	Consider a random vector $\tau$ drawn from such restricted Gaussian distribution, and let $\mu_\tau$ and $\Sigma_\tau$ be the first and second moment respectively. Using \citep[Equation 38 (a-c)]{ghosh2019max}, we have
	\begin{align*}
		\|\mu_\tau\|^2 \leq C \log (1/\nu),
	\end{align*}
	\begin{align*}
		C \nu^2 I_d \preccurlyeq \Sigma_\tau, %\preccurlyeq (1+ C \log (1/\nu) I_d.
	\end{align*}
Moreover \citep[Lemma 15 (a)]{yi2016solving} shows that $\tau$ is subGaussian with $\psi_2$ norm at most $\zeta^2 \leq C(1+\log(1/\pi_{\min})$. Coupled with the definition of $\psi_2$ norm, \cite{vershynin2018high}, we obtain that the centered random variable $ \tau - \mu_{\tau} $ admits a $\psi_2$ norm squared of at most $C_1(1+\log(1/\pi_{\min})$.

With $m$ draws of such random variables,  from \citep[Equation 39]{ghosh2019max}, we have
	\begin{align*}
		\sigma_{\min} \left(  \frac{1}{m} \sum_{i=1}^m \tau_i \tau_i^T \right) \geq C \nu^2 - \zeta^2 \left( \frac{d}{m} + \sqrt{\frac{d}{m}} + \delta \right),
	\end{align*}
	with probability at least $1- 2\exp(-c_1 m \min\{\delta,\delta^2\})$

	If there are $n$ samples from the unrestricted Gaussian distribution, the number of samples, $m$ that fall in $S$ is given by $m \geq \frac{1}{2} \nu n$ with high proibability. This can be seen directly from the binomial tail bounds. We have
	\begin{align*}
		\PP(m \leq \frac{\nu n}{2}) \leq \exp(-c \nu n)
	\end{align*}
	Combining the above, with $\nu \geq c$ where $c$ is a constant as well as $n \geq \frac{C \log (1/\nu)}{\nu^3} d$, we have
	\begin{align*}
		\sigma_{\min} \left(  \frac{1}{m} \sum_{i=1}^m \tau_i \tau_i^T \right) \geq \frac{C}{2} \nu^2,
	\end{align*}
	with probability at least $1- 2\exp(-c_1 m \min\{\delta,\delta^2\})$. Substituting $\delta = C \nu^2$ yields the result.
\end{proof}

\begin{lemma}
	\label{lem:norm}
	Suppose $(x_i,y_i) \in S^*_j$ for some $j \in [k]$. We have
	\begin{align*}
		\|x_i\| \leq C (\sqrt{d \log d \log (1/\pi_{\min})} + \sqrt{\log(1/\pi_{\min})}) \leq C_1 \sqrt{d \log d \log (1/\pi_{\min})},
	\end{align*}
with probability at least $1-1/\mathsf{poly}(d)$, where the degree of the polynomial depends on the constant $C$.
\end{lemma}

\begin{proof}
Note that Lemma~\ref{lem:restricted} shows that under $(x_i,y_i) \in S^*_j$ for some $j \in [k]$, the centered random variable $\tau_i - \mu_\tau$ is sub-Gaussian with $\psi_2$ norm squared of at most $C (1+ \log(1/\pi_{\min}))$. Note that since, $\tau_i - \mu_\tau$ is centered, the $\psi_2$ norm is (orderwise) same as the sub-Gaussian parameter.

We now use the standard norm concentration for sub-Gaussian random variables \cite{jin2019short}. We have, for a sub-Gaussian random vector with parameter at most $C (1+ \log(1/\pi_{\min}))$, we have
	\begin{align*}
		\PP \left (\|X - \EE X\| \geq t \sqrt{d}\sqrt{(1+ \log(1/\pi_{\min})} \right) \leq 2\exp(-c_1 t^2).
	\end{align*} 
Using this with $t = C \sqrt{\log d}$ along with the fact that $\|\mu_\tau\|^2 \leq C \log(1/\pi_{\min})$, we obtain the lemma.
\end{proof}

\section{Proof of Generalization}
\label{app:gen}
\subsection{Proof of Claim~\ref{claim:bounded}}
In order to see this, suppose $h^{(1)}_j \in \cH_j$ and $h^{(2)}_j \in \cH_j$, and so we have $h^{(1)}_j(x) = \inner{x}{\theta^{(1)}_j}$ and $h^{(2)}_j(x) = \inner{x}{\theta^{(2)}_j}$ with $\|\theta^{(1)}_j\| \leq R$ as well as $\|\theta^{(2)}_j\| \leq R$. With this, we have
\begin{align*}
    |\ell(h^{(1)}_j(x),y) - \ell(h^{(2)}_j(x),y)| &= \bigg | \inner{x_i}{\theta^{(2)}_j - \theta^{(1)}_j} [2y - \inner{x}{\theta^{(2)}_j + \theta^{(1)}_j}] \bigg| \\
    & \leq |h^{(1)}_j(x) - h^{(2)}_j(x)| \,\left[ 2|y| + \|x\| (\|\theta^{(1)}_j\| + \|\theta^{(2)}_j\|) \right] \\
    & \leq 2(1+R)\, |h^{(1)}_j(x) - h^{(2)}_j(x)|,
\end{align*}
which proves the claim.

\subsection{Proof of Lemma~\ref{lem:rad}}
\begin{proof}
Note that the soft-min loss is a convex combination of the base losses, and the probabilities are computed by $p_{\theta_1,..,\theta_k}(x,y;\theta_j)$. Instead, if we consider the loss class with \emph{all possible} convex combinations of the base losses, the corresponding loss class will be a superset of the current loss class. From the definition of Rademacher complexity, if $F_1 \subseteq F_2$ for any two sets $F_1$ and $F_2$, we have $\rad_n(F_1) \leq \rad_n(F_2)$. We define the following loss class
\begin{align*}
    \Bar{\Phi} = \bigg \lbrace (x,y) \mapsto  \sum_{j=1}^k \alpha_j \ell(h_j(x),y); \theta_j \in \mathbb{R}^d, \|\theta_j\|\leq R, \alpha_j \geq 0 \forall j \in [k], \sum_{j=1}^k \alpha_j = 1   \bigg \rbrace,
\end{align*}
and hence from the definition of Rademacher complexity, we have
$
    \rad(\Phi) \leq \rad(\Bar{\Phi}).
$
Continuing we have
\begin{align*}
    \rad(\Bar{\Phi}) &= \EE_{\bsigma} \left[ \sup_{\{\theta_j:\|\theta_j\|\leq R ,\alpha_j \geq 0\}_{j=1}^k, \sum_{j=1}^k \alpha_j=1} \,\, \bigg |\frac{1}{n} \sum_{i=1}^{n} \sigma_i \sum_{j=1}^k \alpha_j \ell(h_j(x),y) \bigg | \right] \\
    & = \EE_{\bsigma} \left[ \sup_{\{\theta_j:\|\theta_j\|\leq R ,\alpha_j \geq 0\}_{j=1}^k, \sum_{j=1}^k \alpha_j=1} \,\, \bigg | \sum_{j=1}^k \frac{1}{n} \sum_{i=1}^{n} \sigma_i  \alpha_j \ell(h_j(x),y) \bigg | \right] \\
    & \leq \sum_{j=1}^k \EE_{\bsigma} \left[ \sup_{\theta_j:\|\theta_j\|\leq R ,\alpha_j \geq 0, |\alpha_j|\leq 1} \,\, \bigg | \frac{1}{n} \sum_{i=1}^{n} \sigma_i  \alpha_j \ell(h_j(x),y) \bigg | \right] \\
    & \leq \sum_{j=1}^k \EE_{\bsigma} \left[ \sup_{\theta_j:\|\theta_j\|\leq R ,\alpha_j \geq 0, |\alpha_j|\leq 1} \,\, |\alpha_j| \bigg | \frac{1}{n} \sum_{i=1}^{n} \sigma_i   \ell(h_j(x),y) \bigg | \right] \\
    & \leq \sum_{j=1}^k \EE_{\bsigma} \left[ \sup_{\theta_j:\|\theta_j\|\leq R ,\alpha_j \geq 0, |\alpha_j|\leq 1} \,\,  \bigg | \frac{1}{n} \sum_{i=1}^{n} \sigma_i   \ell(h_j(x),y) \bigg | \right] \\
    & \leq \sum_{j=1}^k \EE_{\bsigma} \left[ \sup_{\theta_j:\|\theta_j\|\leq R} \,\,  \bigg | \frac{1}{n} \sum_{i=1}^{n} \sigma_i   \ell(h_j(x),y) \bigg | \right] \\
    & = k \rad(\ell \circ \mathcal{H}) \\
    & \leq 4k(1+R) \rad(\mathcal{H}) \\
    & \leq \frac{4kR(1+R)}{\sqrt{n}}
\end{align*}
where in the third line, we have used the sub-additivity property of the supremum function as well as the triangle inequality. We also used the above claim regarding the Lipschitz constant of the loss function $\ell(.,.)$ and invoked the contraction result for Rademacher averages by \cite{bartlett2002rademacher}. Finally, for linear hypothesis class, we use \cite{mohri2018foundations} to obtain the final result. 
Hence, we  obtain
\begin{align*}
    \rad(\Phi) \leq \frac{4kR(1+R)}{\sqrt{n}},
\end{align*}
which proves the result.
\end{proof}

\end{document}